\documentclass{article}

\PassOptionsToPackage{numbers,compress}{natbib}

 \usepackage[preprint]{neurips_2026}

\usepackage[utf8]{inputenc} 
\usepackage[T1]{fontenc}    
\usepackage{hyperref}       
\usepackage{url}            
\usepackage{booktabs}       
\usepackage{amsfonts}       
\usepackage{nicefrac}       
\usepackage{microtype}      
\usepackage{xcolor}         

\usepackage{graphicx}
\usepackage{float}
\usepackage{amsmath}
\usepackage{amsthm}
\usepackage{amssymb}
\usepackage{bm}
\usepackage{cleveref}
\usepackage{tikz}
\usepackage{pgfplots}
\pgfplotsset{compat=1.18}
\usepgfplotslibrary{groupplots}
\usetikzlibrary{calc}

\newtheorem{theorem}{Theorem}
\newtheorem{lemma}[theorem]{Lemma}

\newtheorem{proposition}[theorem]{Proposition}
\newtheorem{remark}[theorem]{Remark}

\theoremstyle{definition}

\allowdisplaybreaks

\title{When, why, and how do diffusion posterior samplers fail? A finite-sample lens}

\author{%
  Benjamin A. Burns\\
  Georgia Institute of Technology\\
  \texttt{bburns46@gatech.edu} \\
  \And
  Sara Fridovich-Keil \\
  Georgia Institute of Technology \\
  \texttt{sfk@gatech.edu} \\
}

\begin{document}

\maketitle

\addtocontents{toc}{\protect\setcounter{tocdepth}{-1}}

\begin{abstract}
  
Diffusion models have excellent capacity to model complex distributions of natural data, which has made them a popular and effective choice for posterior sampling in imaging inverse problems.  Existing methods can incorporate any measurement model at inference time but must use an inexact approximation for the likelihood at intermediate timesteps for computational tractability. Although these approximations can often work well empirically, their downstream effect on the sampled posterior is poorly understood and can result in unexplained failures.  To understand when, why, and how these likelihood approximations propagate to erroneous posterior distributions, we introduce a finite-sample perspective on posterior sampling that approximates the posterior to arbitrary precision as training set size tends towards infinity, for any forward model and prior distribution. Using this finite-sample lens, we observe that popular posterior sampling approximations tend to under- or over-estimate the spread of the posterior at intermediate timesteps, causing downstream consequences including sensitivity to early stopping time, inaccurate relative weighting of posterior modes, and hallucination, both of prior modes that are not in the posterior and likelihood modes that are not supported by the prior. Moreover, we find that the cause of these posterior errors requires neither a nonlinear measurement model nor a multimodal posterior, but can arise solely due to a multimodal prior and inaccurate posterior spread at intermediate sampling times. Our finite-sample posterior sampling approach is agnostic to the type of likelihood approximation and the type of (linear or nonlinear) forward model, and can thus serve as a drop-in diagnostic to evaluate the accuracy and failure modes of existing and future posterior samplers. All code for experiments is available at: \href{https://github.com/voilalab/diagnosing-posterior-sampling}{\texttt{https://github.com/voilalab/diagnosing-posterior-sampling}}.

\end{abstract}

\section{Introduction}
\label{sec:introduction}

Computational imaging has become a staple of modern science, medicine, security, and industry, allowing us to computationally visualize inside solid objects and at scales invisible to the human eye. Computational imaging works by first collecting indirect measurements through a known physical \emph{forward model}, and then computationally inverting that forward model to recover an interpretable image of the target object. In cases where the forward model is linear, and we have prior knowledge that constrains the target image to lie in a convex set, the rich theory of compressive sensing guarantees that the true image can be recovered exactly from a minimal number of measurements \citep{adcock2021compressive}.

However, in many computational imaging problems of modern interest, the forward model may be nonlinear, our prior knowledge may be nonconvex or statistical rather than structural, and we may have measurements that are too few or too noisy to uniquely determine the true image. In this context, the \emph{posterior distribution} is the primary object of interest, as it allows us to (i) produce candidate reconstructions via sampling and (ii) quantify confidence of candidate reconstructions via density evaluation, both of which have downstream applications in uncertainty quantification \citep{sullivan2015introduction, ghanem2017handbook}, active measurement acquisition \citep{lin2025tomographic, elata2025psc, goli2024bayes}, and risk-aware decision-making \citep{rockafellar2000optimization, majumdar2017how}. A key challenge in posterior sampling is to accurately capture events or posterior modes which have low probability but high impact, such as observing a tumor in medical imaging or a small dog or child in autonomous driving.

Diffusion models \citep{sohl-dickstein2015deep,ho2020denoising, song2020scorebased} have become an increasingly popular technique for learning a complex distribution from a dataset of samples, and thus an increasingly popular statistical prior for computational imaging \citep{heckel2025deep}.
In particular, score-based methods are an appealing option for posterior sampling because the posterior score, the object required for sampling, decomposes exactly as a sum of a prior score plus a likelihood score. The prior score is learnable from clean image examples, while the likelihood can be written in marginal form given the prior, noise schedule, and forward model. However, these so-called \emph{plug-and-play} \citep{zheng2025inversebench} or \emph{zero-shot} methods \citep{dasgupta2026unifying} cannot directly use this marginal for posterior sampling as evaluating the integral for each sample at each denoising step is computationally tractable.

Moment-matching methods replace the true marginal with a tractable surrogate, specifically by selecting a surrogate for the denoising process of the diffusion model. Although the denoiser is unknown, its moments are accessible through the noise schedule and learned prior score via Tweedie's formula \citep{robbins1956empirical}. However, satisfying both conditions is restrictive. Dirac approximations such as Diffusion Posterior Sampling (DPS) \citep{chung2023diffusion} analytically simplify for all measurement models, but capture no spread information about the data \emph{nor} denoising schedule. Gaussian approximations \citep{song2023pseudoinverseguided, boys2024tweedie} capture additional information, yet only analytically simplify for linear inverse problems with Gaussian measurement noise.

The posterior samples generated under these approximations are often high quality, but can fail dramatically in relatively common imaging contexts such as quantized sensing \citep{xu2024provably}. The root causes of both (i) the unexpectedly strong general performance and (ii) sporadic failure of these approximations are poorly understood. Thus, despite the increasing prevalence of diffusion posterior samplers, relatively little is known about when such methods will fail, why these approximations only fail in isolated examples, and what the characteristic nature of their failure is.

\textbf{Contributions:} To study when, why, and how these existing approaches fail, we introduce a finite-sample perspective on posterior sampling that approximates the posterior to arbitrary precision as training set size tends towards infinity, for \emph{any} forward model and prior. We provide algorithmic analysis and finite-sample rates which inform when the finite-sample perspective gives rise to a tractable posterior sampling method, and leverage this perspective as a lens into the nature of the approximations proposed in prior work and their downstream consequences.
Through this finite-sample lens, we observe that Gaussian approximations \citep{song2023pseudoinverseguided, boys2024tweedie} tend to commit to prior modes too early in the sampling process, which can lead them to place inaccurate relative weight on different posterior modes as well as hallucinate samples from prior modes that are not in the posterior. 
For Dirac approximations (DPS) \cite{chung2023diffusion}, our finite-sample lens reveals that the choice of parameter $\zeta$ causes DPS to over- or under-weight the likelihood relative to the prior at intermediate timesteps during posterior sampling, which can cause erroneous posterior variance at intermediate timesteps, hallucination of prior modes that are inconsistent with the measurement, and/or measurement-consistent modes that are inconsistent with the prior.
These effects can arise even under linear measurement models, Gaussian measurement noise, and unimodal posteriors, as long as the prior distribution is multimodal.

\section{Related work}
\label{sec:brief-related-work}

\vspace{-1mm}
We briefly discuss key related work here, but defer extended discussion to \Cref{sec:generative-posterior-samplers,sec:benchmarking-vs-understanding,sec:exact-diffusion}.
\vspace{-1mm}

\paragraph{Generative posterior samplers.}
Our goal is to analyze the effect of popular moment-matching likelihood (and hence likelihood score) approximations \citep{chung2023diffusion,song2023pseudoinverseguided,boys2024tweedie} on posterior sampling. Alternative methods which employ either learning-based likelihood or posterior scores \citep{elata2025psc} or exact likelihood sampling (e.g. via variable splitting) \citep{xu2024provably} exist. However, we focus on moment-matching methods \cite{chung2023diffusion, song2023pseudoinverseguided, boys2024tweedie} due to two of their key benefits: (i) they can reuse pre-trained prior scores, which are often pre-trained with significantly more resources than are available to train individually for each measurement model, and (ii) the forward model $\mathcal{A}$ can change at inference time, without requiring any additional training. For example, in inpainting, changing the size or location of the mask entirely changes the likelihood and posterior scores, which is a straightforward and efficient adjustment for moment-matching methods but necessitates retraining for methods that learn the likelihood or posterior score.

\paragraph{Benchmarking vs. understanding.}
Many existing works benchmark various diffusion posterior samplers against each other by studying the quality of the resulting posterior samples at terminal time \citep{zheng2025inversebench, crafts2025benchmarking,zhang2025improving}, and provide hypotheses on when and why certain methods perform better than others or struggle in particular contexts \citep{zheng2025inversebench,xu2025rethinking}. However, little is known about how the various methods function at intermediate times, which we demonstrate is key to understanding how and why errors at terminal time arise, and for predicting when the sampler will fail in the future.

\paragraph{Exact diffusion.}
While the finite-sample regime has been used in developing unconditional \citep{scarvelis2025closedform} and conditional \citep{zhang2025exact,zhang2025iensf} sampling algorithms, and for understanding unconditional diffusion \citep{mimikos-stamatopoulos2024scorebased}, we present the first (to our knowledge) finite-sample perspective for understanding posterior sampling.

\section{Preliminaries}
\label{sec:preliminaries}

We begin with a review of Bayesian inverse problems in which posterior sampling is used, diffusion-based generative models that are used to model the prior distribution in these problems, and existing popular approximations to the likelihood score used in posterior sampling with diffusion models.

\paragraph{Notation:} We adopt the convention of using boldface lowercase characters $\mathbf x, \mathbf y$ for (finite-dimensional) vectors, boldface uppercase letters $\mathbf A, \mathbf C$ for matrices, and lowercase non-bold characters $t, \alpha$ for scalars. The typesetting of a function typically denotes the ``type'' of its argument. For example, the noise schedule $\bar\alpha(t): \mathbb{R} \to \mathbb{R}$ and weight function $w_i(\mathbf x, t) : \mathbb{R}^n \times \mathbb{R} \to \mathbb{R}$ are scalar-valued functions as $\alpha, w$ are lowercase non-bold, the denoiser mean $\mathbf m_{0 \mid t}(\mathbf x_t) : \mathbb{R}^n \to \mathbb{R}^n$ is vector-valued, and the denoiser covariance $\mathbf C_{0 \mid t}(\mathbf x_t) : \mathbb{R}^n \to \mathbb{R}^{n\times n}$ is matrix-valued. The main exceptions are that we utilize non-bold calligraphic $\mathcal{A} : \mathbb{R}^n \to \mathbb{R}^m$ to stress situations where the (generally vector-valued) measurement operator is either nonlinear or \emph{possibly} nonlinear, and that we use typical notation for standard probability objects (normal distribution $\mathcal{N}$, expectation $\mathbb{E}$, etc.) irrespective of their output types. We let $\nabla_{\mathbf x} f(\mathbf x)$ and $\nabla^2_{\mathbf x} f(\mathbf x)$ denote the gradient and Hessian, respectively, of scalar-valued function $f : \mathbb{R}^n \to \mathbb{R}$ evaluated at point $\mathbf x$, and let $J_{\mathbf f}(\mathbf x)$ denote the Jacobian of vector-valued function $\mathbf f : \mathbb{R}^n \to \mathbb{R}^m$ evaluated at point $\mathbf x$.

\subsection{Bayesian inverse problems}

A prototypical computational imaging problem can be described by a forward model of the form in \Cref{eq:forward-model}, where $\mathbf x \in \mathcal{X} \subseteq \mathbb{R}^n$ is the true image or signal and $\mathbf y \in \mathbb{R}^m$ is our measurements taken according to a known measurement operator $\mathcal{A}$, with measurement noise $\bm \eta$:
\begin{equation}
      \mathbf y = \mathcal{A}(\mathbf x) + \bm \eta, \quad \bm\eta \sim \mathcal{N}(\mathbf 0, \bm \Sigma_{\mathbf y}).
      \label{eq:forward-model}
\end{equation}
We assume that the measurement operator $\mathcal{A}$ is known and deterministic, but may be linear or nonlinear. For example, in magnetic resonance imaging (MRI), $\mathcal{A}$ is a (subsampled) Fourier transform \citep{mribook}, while in computed tomography (CT) $\mathcal{A}$ involves an exponential nonlinearity due to Beer--Lambert attenuation \citep{ctbook, fridovich2023gradient}. The measurement noise $\bm \eta$ may take various distributions in different imaging problems (for example, in low-dose CT $\bm \eta$ is typically Poisson), but for simplicity of analysis is often assumed to be zero-mean Gaussian \citep{ctbook, heckel2025deep}, as in \Cref{eq:forward-model}. 

Even in this setting where $\mathcal{A}$ and the distribution of $\bm \eta$ are assumed known, solving \Cref{eq:forward-model} for the image $\mathbf x$ remains challenging because (i) the measurement operator $\mathcal{A}$ may be nonlinear, under-determined, non-invertible, or ill-conditioned, and (ii) the measurement noise $\bm \eta$ is stochastic and may obscure relevant signal in the measurements $\mathbf y$. To assist with these challenges, it is common to assume that the signal $\mathbf x$ follows known structure or statistics; the former is standard in compressive sensing \citep{adcock2021compressive} and the latter forms the basis of posterior sampling, which is the focus of this work.

Specifically, we work in the context of \emph{Bayesian} inverse problems, in which we assume data follow a prior distribution $p_{\mathrm{pr}}(\mathbf x)$
and formulate the posterior distribution $p(\mathbf x\mid \mathbf y)$ via Bayes' rule:
\begin{equation}
      p(\mathbf x \mid \mathbf y) = p(\mathbf y \mid \mathbf x)p_{\mathrm{pr}}(\mathbf x) / Z(\mathbf y).
      \label{eq:posterior}
\end{equation}
Crucially, we are interested in recovering the full posterior distribution, to obtain not only the most likely samples but also an accurate representation of uncertainty over the range of possible reconstructions that are consistent with both the prior and the measurements. The likelihood $p(\mathbf y\mid \mathbf x)$ in \Cref{eq:posterior} is induced directly by the forward model in \Cref{eq:forward-model}, and accounts for the measurement operator $\mathcal{A}$ and measurement noise $\bm \eta$.
Thus, what remains is to characterize the prior distribution $p_{\mathrm{pr}}$ and the normalization constant $Z(\mathbf y)$, as we describe next.

\subsection{Diffusion-based generative models}

To characterize the prior distribution $p_{\mathrm{pr}}$, we turn to generative models, which have shown strong capacity to capture the rich multimodal structure of natural data distributions. In particular, score-based diffusion models are rather natural for solving Bayesian inference problems, because the score of the posterior distribution has no dependence on the normalization constant $Z(\mathbf y)$.

Following prior work \citep{chung2023diffusion, song2023pseudoinverseguided, boys2024tweedie},
we consider the variance-preserving (VP) stochastic differential equation (SDE) formulation \citep{song2020scorebased}, which progressively noises samples $\mathbf x_0 \sim p_{\mathrm{pr}}$ via the (rescaled) Ornstein-Uhlenbeck process \citep{uhlenbeck1930theory} :
\begin{equation}
      \label{eq:vpsde}
      \mathrm{d}\mathbf x_t = \underbrace{-\tfrac{1}{2}\beta(t)\mathbf x_t}_{\text{drift}} \, \mathrm{d}t + \underbrace{\sqrt{\beta(t)}}_{\text{diffusion}}\, \mathrm{d}\mathbf{w}_t, \quad \mathbf x_0 \sim p_{\mathrm{pr}},
\end{equation}
where $\beta \colon [0, \infty) \to \mathbb{R}_{>0}$ is a positive function, $\mathbf{w}_t$ is a standard Wiener process, and $\mathbf x_t$ denotes the evolution of sample $\mathbf x_0$ subject to \Cref{eq:vpsde} for $t$ time. Importantly, \Cref{eq:vpsde} is linear in the state $\mathbf x_t$, which allows us to analytically compute the forward transition to any positive time $t$:
\begin{equation}
      p_{t\mid 0}(\mathbf x_t \mid \mathbf x_0)
      = \mathcal{N}\Bigl(\mathbf x_t; \sqrt{\bar\alpha(t)}\mathbf x_0, \bigl(1 - \bar\alpha(t)\bigr)\mathbf I_n\Bigr),\quad\text{where}\quad\bar\alpha(t) := \exp\left(-\int_0^t\beta(\tau)\mathrm{d}\tau\right).\label{eq:def-alpha-bar}
\end{equation}
Observe that when $t = 0$, the integral in \Cref{eq:def-alpha-bar} is zero, hence $\bar \alpha (0) = 1$. Additionally, assuming that $\beta(t)$ is chosen such that the integral in \Cref{eq:def-alpha-bar} becomes arbitrarily large for large $t$, we have $\bar\alpha(t) \to 0$ as $t \to \infty$. Thus, the forward transition kernel $p_{t \mid 0}(\mathbf x_t\mid \mathbf x_0)$ converges to a standard normal $\mathcal{N}(\mathbf x_t ; \mathbf 0, \mathbf I_n)$ as $t\to \infty$, for all $\mathbf x_0$.  Furthermore, by \citet{anderson1982reversetime} the reverse diffusion process is analytically computable:
\begin{equation}
      \mathrm{d}\mathbf x_t = \left[-\tfrac12 \beta(t)\mathbf x_t - \beta(t)\nabla_{\mathbf x_t} \log p_t(\mathbf x_t)\right] \mathrm{d}t + \sqrt{\beta(t)}\, \mathrm{d}\bar{\mathbf{w}}_t,
      \label{eq:backward-process}
\end{equation}
where $\bar{\mathbf{w}}_t$ is a standard reverse-time Wiener process. The reverse process maps the unit Gaussian (at $t \to \infty$) to $p_{\mathrm{pr}}$ (at $t = 0$), thus given access to the score function $\nabla_{\mathbf x_t} \log p_t(\mathbf x_t)$ we can (i) draw Gaussian noise i.i.d.~to initialize $\mathbf x_T$ at a large time $T$, and (ii) numerically simulate the reverse process to time $t = 0$ (for example, via the Euler-Maruyama scheme) to produce samples from $p_{\mathrm{pr}}$. 

Importantly, while diffusion models let us draw samples from complex distributions, they do not give us access to the distribution (density) itself. This distinguishes diffusion models from normalizing flows \citep{rezende2015variational}, which employ (restrictive) reversible and differentiable flow architectures that provide density access via change of variables. However, the utility of diffusion models is that score-based methods avoid computation of normalization constants by learning the score function via score matching \citep{hyvarinen2005estimation,vincent2011connection}.

\subsection{Diffusion posterior sampling via moment matching}

The most direct approach for applying a diffusion model to sample from the posterior distribution $p(\mathbf x \mid \mathbf y)$ introduced in \Cref{eq:posterior} would be to drive the reverse process in \Cref{eq:backward-process} using the posterior score
$\nabla_{\mathbf x_t} \log p_{t \mid \mathbf y}(\mathbf x_t \mid \mathbf y)$ instead of the prior score $\nabla_{\mathbf x_t} \log p_{t}(\mathbf x_t)$:
\begin{equation}
      \mathrm{d}\mathbf x_t = \left[-\tfrac12 \beta(t)\mathbf x_t - \beta(t)\nabla_{\mathbf x_t} \log p_{t \mid \mathbf y}(\mathbf x_t \mid \mathbf y)\right] \mathrm{d}t + \sqrt{\beta(t)}\, \mathrm{d}\bar{\mathbf{w}}_t.
      \label{eq:posterior-backward-process}
\end{equation}
A key merit of this approach is that, by our Bayes' rule categorization of the posterior in \Cref{eq:posterior}, evaluating the posterior score requires knowledge of only the likelihood score and prior score, but not of the normalization constant $Z(\mathbf y)$:
\begin{equation}
\begin{split}
      \nabla_{\mathbf x_t} \log p_{t \mid \mathbf y}(\mathbf x_t \mid \mathbf y) &= \nabla_{\mathbf x_t} \log p_{\mathbf y \mid t}(\mathbf y \mid \mathbf x_t) + \nabla_{\mathbf x_t} \log p_t(\mathbf x_t) \label{eq:score-identity}  - \underbrace{\nabla_{\mathbf x_t} \log Z(\mathbf y)}_{0}.
\end{split}
\end{equation}
A pre-trained prior score $\nabla_{\mathbf x_t} \log p_t(\mathbf x_t)$ in \Cref{eq:score-identity} is often available from a diffusion model trained to generate samples from $p_{\mathrm{pr}}$, for example high-resolution natural images. The unknown denoising likelihood $p_{\mathbf y \mid t}(\mathbf y \mid \mathbf x_t)$ can be written in marginal form
\begin{equation}
    p_{\mathbf y \mid t}(\mathbf y \mid \mathbf x_t)
    = \int_{\mathcal X} \underbrace{p_{\mathbf y\mid 0, t}(\mathbf y \mid \mathbf x_0, \mathbf x_t)}_{\text{Markov property}} p_{0\mid t}(\mathbf x_0 \mid \mathbf x_t) \mathrm{d} \mathbf x_0
    = \int_{\mathcal X} p_{\mathbf y \mid 0}(\mathbf y \mid \mathbf x_0)\, p_{0 \mid t}(\mathbf x_0 \mid \mathbf x_t)\mathrm{d} \mathbf x_0 , \label{eq:true-marginal}
\end{equation}
but cannot be evaluated analytically in general as (i) the denoiser $p_{0 \mid t}(\mathbf x_0 \mid \mathbf x_t)$ is not typically known, and (ii) the integral does not typically analytically simplify, even given knowledge of $p_{0 \mid t}$. Furthermore, estimating the latter integral in \Cref{eq:true-marginal} via Monte Carlo integration is prohibitively expensive in high dimensions because it requires drawing many samples from the denoiser, each of which requires running the reverse diffusion process from time $t$ to time $0$.

Moment-matching methods \citep{chung2023diffusion,song2023pseudoinverseguided,boys2024tweedie} approximate the marginal integral by approximating the denoiser $p_{0 \mid t}(\mathbf x_0 \mid \mathbf x_t)$ in a way that (i) allows the integral to simplify analytically, (ii) approximates the denoiser well, and (iii) is accessible given only knowledge of the unconditional marginal score $\nabla_{\mathbf x_t} \log p_t(\mathbf x_t)$ and noise schedule $\bar\alpha(t)$. The core differences between moment-matching methods is what specific denoiser approximation is chosen, and what assumptions are made about the forward model such that the approximated marginal analytically simplifies.

The name ``moment-matching'' stems from the technique of picking a simple distribution (e.g., a single Dirac delta or a unimodal Gaussian) whose first and/or second moments are equal to the moments of the denoiser. Moment-matching techniques are tractable because the denoiser's moments are computable using the (unconditional) marginal score and noise schedule via Tweedie's formula \citep{robbins1956empirical} since the forward transition $p_{t \mid 0}$ is an exponential family, requiring no additional knowledge or assumptions on the prior or denoiser's distribution class.

\begin{proposition}[due to \citep{boys2024tweedie}]
    Let $\mathbf m_{0 \mid t}(\mathbf x_t)$ and $\mathbf C_{0 \mid t}(\mathbf x_t)$ denote the mean and covariance of the denoiser $p_{0 \mid t}(\cdot \mid \mathbf x_t)$. Given knowledge of marginal score $\nabla_{\mathbf x_t} \log p_t(\mathbf x_t)$ and noise schedule $\bar\alpha(t)$, the mean and covariance are computable as
    \begin{equation}
        \mathbf m_{0 \mid t}(\mathbf x_t) := \mathbb{E}[\mathbf x_0 \mid \mathbf x_t] = \frac{1}{\sqrt{\bar \alpha(t)}}
      \Bigl(\mathbf x_t + \bigl(1 - \bar\alpha(t)\bigr)\nabla_{\mathbf x_t} \log p_t(\mathbf x_t)\Bigr),
      \label{eq:tweedies-formula}
    \end{equation}
    \begin{equation}
    \begin{split}
         \mathbf C_{0\mid t}(\mathbf x_t) 
        := \operatorname{Cov}(\mathbf x_0 \mid \mathbf x_t)
        = \tfrac{1 - \bar\alpha(t)}{\bar\alpha(t)} (\mathbf I_n + (1 - \bar\alpha(t))\nabla^2_{\mathbf x_t}\log p_t(\mathbf x_t))
        = \tfrac{1 - \bar\alpha(t)}{\sqrt{\bar\alpha(t)}} J_{\mathbf m_{0\mid t}}(\mathbf x_t).
        \label{eq:tweedies-covariance}
    \end{split}
    \end{equation}
\end{proposition}
\begin{proof}
    See Appendix A.1 of \citet{boys2024tweedie}.
\end{proof}
\begin{remark}
    We adopt the $\mathbf m_{0 \mid t}$ and $\mathbf C_{0 \mid t}$ notation from \citet{boys2024tweedie} for ease of comparison. However, we emphasize the dependence of $\mathbf m_{0 \mid t}(\mathbf x_t)$ and $\mathbf C_{0 \mid t}(\mathbf x_t)$ on $\mathbf x_t$ by including the state as an argument of the estimate.
\end{remark}

\paragraph{Dirac approximations}
\citet{chung2023diffusion} (DPS) apply a first-moment denoiser approximation, replacing the true denoiser $p_{0 \mid t}(\mathbf x_0 \mid \mathbf x_t)$ with a Dirac distribution at the denoiser mean $\mathbf m_{0 \mid t}(\mathbf x_t)$:
\begin{align}
      \int_{\mathcal X} p_{\mathbf y \mid 0}(\mathbf y \mid \mathbf x_0)\, p_{0 \mid t}(\mathbf x_0 \mid \mathbf x_t)\mathrm{d} \mathbf x_0 
      &\approx \int_{\mathcal X} p_{\mathbf y \mid 0}(\mathbf y \mid \mathbf x_0)\, \delta(\mathbf x_0 - \mathbf m_{0 \mid t}(\mathbf x_t))\mathrm{d} \mathbf x_0 . \label{eq:chung-approximation} 
\end{align}
Although the integral in \Cref{eq:chung-approximation} simplifies for all measurement operators and noise models (so long as $p_{\mathbf y \mid 0}(\mathbf y \mid \cdot)$ is measurable and finite), the resulting likelihood adopts the distribution class of the measurement noise. 
For example, when the measurement noise $\bm \eta \sim \mathcal{N}(\mathbf 0, \bm \Sigma_{\mathbf y})$ is additive Gaussian, \Cref{eq:chung-approximation} is a unimodal Gaussian
\begin{align}
    \int_{\mathcal X} p_{\mathbf y \mid 0}(\mathbf y \mid \mathbf x_0)\, \delta(\mathbf x_0 - \mathbf m_{0 \mid t}(\mathbf x_t))\mathrm{d} \mathbf x_0
    &= \int_{\mathcal X} \mathcal{N}\bigl(\mathbf y; \mathcal{A}(\mathbf x_0), \bm \Sigma_{\mathbf y}\bigr)\, \delta(\mathbf x_0 - \mathbf m_{0 \mid t}(\mathbf x_t))\mathrm{d} \mathbf x_0  \\
    &= \mathcal{N}\bigl(\mathbf y; \mathcal{A}(\mathbf m_{0 \mid t}(\mathbf x_t)), \bm \Sigma_{\mathbf y}\bigr) \label{eq:chung-gaussian}
\end{align}
whose score is computable via the chain rule and \Cref{eq:tweedies-covariance}:
\begin{align}
    \nabla_{\mathbf x_t} \log \mathcal{N}\bigl(\mathbf y; \mathcal{A}(\mathbf m_{0 \mid t}(\mathbf x_t)), \bm \Sigma_{\mathbf y}\bigr)
    &= -\nabla_{\mathbf x_t}\left\|\mathbf y - \mathcal{A}(\mathbf m_{0 \mid t}(\mathbf x_t))\right\|^2_{\bm\Sigma_{\mathbf y}^{-1}} \\
    &= J_{\mathbf m_{0\mid t}}(\mathbf x_t)^\top J_{\mathcal{A}}(\mathbf m_{0 \mid t}(\mathbf x_t))^\top\bm \Sigma_{\mathbf y}^{-1} (\mathbf y - \mathcal{A}(\mathbf m_{0 \mid t}(\mathbf x_t))) \\
    &= \tfrac{\sqrt{\bar\alpha(t)}}{1 - \bar\alpha(t)}\mathbf C_{0 \mid t}(\mathbf x_t) J_{\mathcal{A}}(\mathbf m_{0 \mid t}(\mathbf x_t))^\top\bm \Sigma_{\mathbf y}^{-1} (\mathbf y - \mathcal{A}(\mathbf m_{0 \mid t}(\mathbf x_t))). \label{eq:chung-score}
\end{align}

\paragraph{Gaussian approximations}
Gaussian denoiser approximations \citep{song2023pseudoinverseguided,boys2024tweedie} improve on the first-moment approximation by additionally allowing for denoiser covariance approximation:
\begin{equation}\label{eq:gaussian-marginal}
    \int_{\mathcal{X}} p_{\mathbf y \mid 0}(\mathbf y \mid \mathbf x_0) p_{0 \mid t}(\mathbf x_0 \mid \mathbf x_t) \mathrm{d}\mathbf x_0 \approx \int_{\mathcal{X}} p_{\mathbf y \mid 0}(\mathbf y \mid \mathbf x_0)\, \mathcal{N}\bigl(\mathbf x_0; \mathbf m_{0 \mid t}(\mathbf x_t), \bm \Sigma_t(\mathbf x_t) \bigr) \mathrm{d}\mathbf x_0,
\end{equation}
where $\bm \Sigma_t(\mathbf x_t)$ denotes some user-chosen, time- and/or state-dependent covariance estimate. While Gaussian approximations can encode strictly more information than first-moment Dirac approximations, their applicability is more limited. The resulting marginal integral in \Cref{eq:gaussian-marginal} only simplifies when the measurement likelihood $p_{y \mid 0}$ is Gaussian affine, i.e. when the measurement operator $\mathcal{A}(\mathbf x_0) = \mathbf A \mathbf x_0 + \mathbf b$ is affine and the measurement noise $\bm \eta \sim \mathcal{N}(\mathbf 0, \bm \Sigma_{\mathbf y})$ is additive Gaussian:
\begin{equation}
    \int_{\mathcal{X}} p_{\mathbf y \mid 0}(\mathbf y \mid \mathbf x_0)\, \mathcal{N}\bigl(\mathbf x_0; \mathbf m_{0 \mid t}(\mathbf x_t), \bm \Sigma_t(\mathbf x_t) \bigr) \mathrm{d}\mathbf x_0
    = \mathcal{N}\bigl(\mathbf y; \mathbf A \mathbf m_{0 \mid t}(\mathbf x_t) {+} \mathbf b, \bm \Sigma_{\mathbf y} {+} \mathbf A \bm \Sigma_t(\mathbf x_t) \mathbf A^\top \bigr). \label{eq:gaussian-approximation}
\end{equation}
Similar to \Cref{eq:chung-score}, the resulting likelihood score can be computed via chain rule and \Cref{eq:tweedies-covariance} as:
\begin{align}
    \nabla_{\mathbf x_t} \log p_{\mathbf y \mid t}(\mathbf y \mid \mathbf x_t) 
    &= J_{\mathbf m_{0\mid t}}(\mathbf x_t)^\top J_{\mathbf A}(\mathbf m_{0 \mid t}(\mathbf x_t))^\top (\bm \Sigma_{\mathbf y} + \mathbf A \bm \Sigma_t(\mathbf x_t) \mathbf A^\top)^{-1}(\mathbf y - \mathbf A \mathbf m_{0 \mid t} - \mathbf b).\\
    &= \tfrac{\sqrt{\bar\alpha(t)}}{1 - \bar\alpha(t)}\mathbf C_{0 \mid t}(\mathbf x_t) \mathbf A^\top (\bm \Sigma_{\mathbf y} + \mathbf A \bm \Sigma_t(\mathbf x_t) \mathbf A^\top)^{-1}(\mathbf y - \mathbf A \mathbf m_{0 \mid t} - \mathbf b).
    \label{eq:linear-likelihood-score}
\end{align}

\citet{song2023pseudoinverseguided} inflate the covariance of the denoiser approximation by a time-dependent scalar factor $\bm \Sigma_t(\mathbf x_t) = r_t^2\mathbf I_n$, which is chosen offline using known information about likely prior data $\mathbf x_0$ and the noise schedule $\bar\alpha(t)$ \citep[see][Appendix A.3 for discussion on choosing $r_t$]{song2023pseudoinverseguided}:
\begin{equation}\label{eq:song-approximation}
    p_{\mathbf y \mid t}(\mathbf y \mid \mathbf x_t) = \mathcal{N}\bigl(\mathbf y; \mathbf A \mathbf m_{0 \mid t}(\mathbf x_t) + \mathbf b, \bm \Sigma_{\mathbf y} + r_t^2\mathbf A \mathbf A^\top \bigr).
\end{equation}
\begin{remark}
    \citet{song2023pseudoinverseguided} extend their algorithm to nonlinear $\mathcal{A}$ when the measurement process is noiseless (i.e., $\bm \eta \equiv \mathbf 0$) by finding a generalized inverse $\mathcal{A}^\dagger$ to the measurement operator satisfying $(\mathcal{A} \circ \mathcal{A}^\dagger \circ \mathcal{A})(\mathbf x) = \mathcal{A}(\mathbf x)$. However, because we consider noisy inverse problems, our analysis of \citet{song2023pseudoinverseguided} is restricted to the case of affine measurement operators.
\end{remark}

\citet{boys2024tweedie} select the Gaussian approximation closest to the denoiser in KL divergence, which is given
by setting $\bm\Sigma_{t}$ equal to the denoiser's true covariance $\mathbf C_{0\mid t}$ \citep[see][Proposition 2]{boys2024tweedie}, which can again be computed via Tweedie's formula \citep[see][Appendix A.1]{boys2024tweedie}:
\begin{equation}\label{eq:boys-approximation}
    p_{\mathbf y \mid t}(\mathbf y \mid \mathbf x_t) = \mathcal{N}\bigl(\mathbf y; \mathbf A \mathbf m_{0 \mid t}(\mathbf x_t) + \mathbf b, \bm \Sigma_{\mathbf y} + \mathbf A \mathbf C_{0\mid t}(\mathbf x_t) \mathbf A^\top \bigr).
\end{equation}
In the case where the prior $p_{\mathrm{pr}}$ is Gaussian, the denoiser $p_{0\mid t}$ is truly Gaussian, and so \Cref{eq:boys-approximation} is exact. If we compare \Cref{eq:boys-approximation} to the approximation of \citet{chung2023diffusion} in \Cref{eq:chung-gaussian}, we see that the latter approach is missing the additive, time-dependent $\mathbf A \mathbf C_{0\mid t}(\mathbf x_t) \mathbf A^\top$ covariance factor, and instead adopts the measurement noise covariance $\bm \Sigma_{\mathbf y}$ across all denoiser times (and hence noise levels).

Again, while the Gaussian approximations encode strictly more information about the denoising process compared to the first-moment approximation of \citet{chung2023diffusion} (\Cref{eq:chung-approximation}), their usage is restricted to measurement models (\Cref{eq:forward-model}) with affine measurement operator $\mathcal{A}(\mathbf x) = \mathbf A \mathbf x + \mathbf b$ and additive Gaussian measurement noise $\bm \eta \sim \mathcal{N}(\mathbf 0, \bm \Sigma_{\mathbf y})$ with \emph{isotropic} covariance $\bm \Sigma_{\mathbf y} = \sigma_{\mathbf y}^2 \mathbf I_m$.
Although the Dirac approximation \citep{chung2023diffusion} can be used when the measurement noise $\bm \eta$ is non-Gaussian and/or data-dependent, we limit our analysis to the additive Gaussian noise case for sake of comparison. There may be additional impacts of the first-moment approximation when the noise is non-Gaussian and/or data-dependent, the analysis of which we leave for future work.

\subsection{Practical diffusion posterior sampling}

Although the Dirac and Gaussian denoiser approximations lead to closed form likelihoods and likelihood scores, additional approximations are often introduced in practice to either reduce computational cost or improve reconstruction quality. For example, \citet{chung2023diffusion} modify the score derived from their Dirac approximation 
\begin{equation}
    \nabla_{\mathbf x_t} \log \mathcal{N}\bigl(\mathbf y; \mathcal{A}(\mathbf m_{0 \mid t}(\mathbf x_t)), \bm \Sigma_{\mathbf y}\bigr)
    = - \frac{1}{\sigma_{\mathbf y}^2}\nabla_{\mathbf x_t} \|\mathbf y - \mathcal{A}(\mathbf m_{0 \mid t}(\mathbf x_t))\|_2^2,
\end{equation} 
by replacing the constant, data-independent measurement noise prefactor $1 / \sigma_{\mathbf y}^2$ with a data-dependent step size $\rho(\mathbf x_t)$, most commonly a scalar $\zeta$ divided by the measurement residual \citep[see][page 6, footnote 5]{chung2023diffusion}:
\begin{align}
    \nabla_{\mathbf x_t} \log \mathcal{N}\bigl(\mathbf y; \mathcal{A}(\mathbf m_{0 \mid t}(\mathbf x_t)), \bm \Sigma_{\mathbf y}\bigr) 
    &\approx -\rho(\mathbf x_t) \nabla_{\mathbf x_t} \|\mathbf y - \mathcal{A}(\mathbf m_{0 \mid t}(\mathbf x_t))\|_2^2, \\
    &= -\frac{\zeta}{\|\mathbf y - \mathcal{A}(\mathbf m_{0 \mid t}(\mathbf x_t))\|_2} \nabla_{\mathbf x_t} \|\mathbf y - \mathcal{A}(\mathbf m_{0 \mid t}(\mathbf x_t))\|_2^2. \label{eq:modified-chung-score}
\end{align}
This is equivalent to, in addition to the Dirac approximation of the denoiser, approximating the Gaussian measurement likelihood $p_{\mathbf y \mid 0}(\mathbf y \mid \mathbf x_0) = \mathcal{N}(\mathbf y; \mathcal{A}(\mathbf x_0), \sigma_{\mathbf y}^2 \mathbf I_m)$ with a Laplace distribution $p_{\mathbf y \mid 0}(\mathbf y \mid \mathbf x_0) \propto \exp\bigl(-2\zeta \|y - \mathcal{A}(\mathbf x_0)\|_2\bigr)$.
The authors state that the scalar $\zeta$ is a problem-dependent hyperparameter which must be tuned, with different $\zeta$ values producing different reconstructions \citep[see][Figure 8, Appendix C.2, and Appendix D.1]{chung2023diffusion}. Additional works have studied the effect that the choice of $\zeta$ has on reconstructed image quality \citep{boys2024tweedie}. The details of our $\zeta$ tuning can be found in \Cref{sec:zeta-tuning}.

\citet{song2023pseudoinverseguided} primarily demonstrate their method for noiseless inverse problems. Among noisy inverse problems, \citet{song2023pseudoinverseguided} solely consider measurement operators admitting matrix representations $\mathbf A \in \mathbb{R}^{m \times n}$ with orthonormal rows $\mathbf A \mathbf A^\top = \mathbf I_m$, and hence are precisely the class of measurement operators where the matrix inversion necessary for evaluating the likelihood score in \Cref{eq:linear-likelihood-score} reduces to scalar division:
\begin{equation}
    (\bm \Sigma_{\mathbf y} + \mathbf A \bm \Sigma_t(\mathbf x_t) \mathbf A^\top)^{-1}
    = (\sigma_{\mathbf y}^2 \mathbf I_m + r_t^2\underbrace{\mathbf A \mathbf A^\top}_{\mathbf I_m})^{-1}
    = (\sigma_{\mathbf y}^2 + r_t^2)^{-1}.
\end{equation}
\citet{song2023pseudoinverseguided} do not describe the computational cost of their method for high-dimensional, noisy inverse problems where $\mathbf A \mathbf A^\top \neq \mathbf I_m$,  or how this inverse is tractably computed or approximated. 

\citet{boys2024tweedie} in general aim to use the full $\mathbf A \mathbf C_{0 \mid t}(\mathbf x_t) \mathbf A^\top$ despite the prohibitive cost of (i) repeatedly assembling this term for each $\mathbf x_t$ and (ii) inverting the full likelihood covariance to compute the likelihood score. 
One option presented by \citet{boys2024tweedie} employs the conjugate gradient (CG) method previously applied by \citet{rozet2024learning} for non-sparse $\mathcal{A}$, for example accelerated MRI \citep[][Appendix E.1]{boys2024tweedie}. Alternatively, the covariance can be heuristically approximated to decrease computational cost, namely by decreasing the necessary number of Jacobian--vector products per iteration. For example, the full denoiser covariance may be replaced by its diagonal $\mathbf C_{0 \mid t}(\mathbf x_t) \approx \operatorname{diag}\bigl(\mathbf C_{0 \mid t}(\mathbf x_t)\bigr)$, in effect assuming that the dimensions of the denoiser $p_{0 \mid t}$ are independent. For specific imaging problems $\mathbf A$ such as inpainting, the diagonal may be replaced by the Jacobian's row sum
\citep[see][Section 3.4]{boys2024tweedie}.

Despite their approximations, these moment-matching have been widely and effectively used across a wide range of inverse problems, both in imaging \citep{heckel2025deep} and in physical sciences \citep{yao2026guided, wang2026fundiff}. However, they have been shown to suffer unexpected failures in certain imaging problems \citep{xu2024provably}, and the effects of their approximations remain understudied, such that these failures are difficult to predict.
\citet{chung2023diffusion} provide an error bound on the approximated likelihood \emph{density}, but (i) do not identify when this error bound is small, and (ii) do not bound the error in the likelihood \emph{score}, which is what drives posterior sampling. Moreover, even given understanding of approximation error in the likelihood score, it is unclear how this error propagates to the posterior samples, which are iteratively influenced by both the likelihood score and the prior score.

To understand the impact of these popular approximations on the posterior samples, we need a way to compute the exact posterior $p_{0 \mid \mathbf y}(\mathbf x_0 \mid \mathbf y)$. To understand how and why particular errors manifest and arise, we also need a way to compute the intermediate-timestep posterior $p_{t \mid \mathbf y}(\mathbf x_t \mid \mathbf y)$. We show that both of these quantities are computable in the finite-sample regime, which we describe next.

\section{Finite-sample diffusion posterior sampling}
\label{sec:finite-sample-regime}
In the \emph{finite-sample regime}, we assume that the prior distribution $p_{\mathrm{pr}}$ is a discrete measure comprised of $N$ i.i.d. samples, rather than a continuous distribution \emph{learned} from these samples. Studying posterior sampling in this regime has two key advantages: (i) as we show next, it allows us to compute the posterior analytically at any timestep $t$, and (ii) it allows us to study the approximation errors induced by the likelihood approximations of \citep{chung2023diffusion, song2023pseudoinverseguided, boys2024tweedie}.

Let $p_{\mathrm{pr}}$ denote the prior distribution from which we are provided i.i.d.\ samples $\{\mathbf x^{(i)}\}_{i=1}^{N} \sim p_{\mathrm{pr}}$, whose \emph{empirical distribution} we denote as
\begin{equation}
      \label{eq:def-empirical-distribution}
      p_{\mathrm{pr}}^N(\mathbf x_t) := \frac{1}{N} \sum_{i = 1}^{N} \delta(\mathbf x_t -\mathbf x^{(i)}).
\end{equation}
The utility of studying the empirical distribution is that it gives us a closed form approximation for the prior distribution $p_{\mathrm{pr}}^N \approx p_{\mathrm{pr}}$ whose positive-time marginals $p_{t}^N(\mathbf x_t)$ are computable by linearity of the VP-SDE dynamics (\Cref{eq:vpsde}). By Bayes' rule, this allows us to analytically compute the denoiser $p_{0 \mid t}^N(\mathbf x_0 \mid \mathbf x_t)$, which is inaccessible in the density setting, as stated in \Cref{lem:finite-sample-diffusion}.
\begin{proposition}[finite-sample diffusion]
      \label[proposition]{lem:finite-sample-diffusion}
      The marginal distribution $p_t^N$ at time $t \geq 0$ of the variance-preserving SDE defined in \Cref{eq:vpsde} acting on empirical distribution $p_{\mathrm{pr}}^N$ is a Gaussian mixture
      \begin{equation*}
            p_t^N(\mathbf x_t) = \frac{1}{N}\sum_{i = 1}^{N} \mathcal{N}
            \Bigl(\mathbf x_t; \sqrt{\bar\alpha(t)}\mathbf x^{(i)}, \bigl(1 - \bar \alpha(t)\bigr)\mathbf I_n\Bigr).
      \end{equation*}
      Furthermore, the denoising distribution $p_{0\mid t}^N$ is a discrete measure
      \begin{equation*}
            p_{0\mid t}^N(\mathbf x_0 \mid \mathbf x_t) = \sum_{i = 1}^{N} w_i(\mathbf x_t, t) \delta(\mathbf x_0 - \mathbf x^{(i)}), \;\,\,
            w_i(\mathbf x_t, t) :=
            \frac{\mathcal{N}\Bigl(\mathbf x_t; \sqrt{\bar\alpha(t)}\mathbf x^{(i)}, \bigl(1 - \bar \alpha(t)\bigr)\mathbf I_n\Bigr)}{\sum_{j = 1}^{N}\mathcal{N}\Bigl(\mathbf x_t; \sqrt{\bar\alpha(t)}\mathbf x^{(j)}, \bigl(1 - \bar \alpha(t)\bigr)\mathbf I_n\Bigr)}.
      \end{equation*}
\end{proposition}
\begin{proof}
      See \Cref{proof:finite-sample-marginal}.
\end{proof}

First, recall that $\bar\alpha(t) \searrow 0$ as $t \to \infty$, thus the mixture weights $w_i(\mathbf x, t)$ converge to uniform for all $\mathbf x \in \mathbb{R}^n$ for sufficiently large $t$. Second, observe that the denoiser $p^N_{0 \mid t}(\mathbf x_0 \mid \mathbf x_t)$ assigns probability zero to all $\mathbf x_0$ outside our dataset, and therefore does not generalize without additional modifications, such as by smoothing \citep{scarvelis2025closedform}.
However, the discrete structure of the denoiser is what makes the finite-sample regime useful for analyzing posterior sampling.

\subsection{Posterior sampling in the finite-sample regime}

In the standard diffusion posterior sampling setting (\Cref{eq:true-marginal}), we were unable to evaluate the marginal
\begin{equation}
    p_{\mathbf y \mid t}(\mathbf y \mid \mathbf x_t)
    = \int_{\mathcal{X}} \mathcal{N}\left(\mathbf y; \mathcal{A}(\mathbf x_0), \bm\Sigma_{\mathbf y}\right)\, p_{0 \mid t}(\mathbf x_0 \mid \mathbf x_t) \mathrm{d} \mathbf x_0
\end{equation}
due to lack of knowledge about the denoiser $p_{0 \mid t}(\mathbf x_0 \mid \mathbf x_t)$, motivating either (i) first-order Dirac approximations \citep{chung2023diffusion} which fail to capture higher moments or (ii) Gaussian approximations \cite{song2023pseudoinverseguided, boys2024tweedie} which only apply when $\mathcal{A}$ is affine. However, in the finite-sample regime the denoiser is a known discrete measure by \Cref{lem:finite-sample-diffusion}, giving us a closed-form observation likelihood even when $\mathcal{A}$ is nonlinear, as described in \Cref{thm:finite-sample-likelihood}.

\begin{proposition}[finite-sample likelihood]
      \label[proposition]{thm:finite-sample-likelihood}
      Given noised sample $\mathbf x_t$ at time $t$ produced by the variance-preserving SDE defined in \Cref{eq:vpsde} acting on empirical distribution $p_{\mathrm{pr}}^N$,
      the likelihood $p_{\mathbf y \mid t}^N$ of observational data $\mathbf y$ is a Gaussian mixture
      \begin{equation*}
            p_{\mathbf y \mid t}^N(\mathbf y \mid \mathbf x_t) = \sum_{i = 1}^{N} w_i(\mathbf x_t, t)\, \mathcal{N}\Bigl(\mathbf y; \mathcal{A}\bigl(\mathbf x^{(i)}\bigr), \bm\Sigma_{\mathbf y}\Bigr).
      \end{equation*}
\end{proposition}
\begin{proof}
      See \Cref{proof:finite-sample-likelihood}.
\end{proof}

To understand the action of this likelihood, observe that the weight function $w_i(\mathbf x, t)$ solely disincentivizes $\mathbf x_t$ from denoising to training data whose expected position $\sqrt{\bar\alpha(t)} \mathbf x^{(i)}$ at time $t$ is far away, having nothing to do with the measurement $\mathbf y$. In contrast, the Gaussian density $\mathcal{N}(\mathbf y; \mathcal{A}(\mathbf x^{(i)}), \bm\Sigma_{\mathbf y})$ disincentivizes $\mathbf x_t$ from denoising to training data which are unlikely to produce $\mathbf y$ under the forward process, having nothing to do with $\mathbf x_t$'s current position. Thus, the multiplicative interaction between the two terms incentivizes the diffusion model to denoise state $\mathbf x_t$ towards training data points $\mathbf x^{(i)}$ which both (i) have high probability of noising to $\mathbf x_t$ under the (unconditional) forward process (\Cref{eq:vpsde}) and (ii) produce measurements consistent with the provided measurement $\mathbf y$.

Finally, because the denoising likelihood $p^N_{\mathbf y \mid t}$ and marginal $p^N_{t}$ are both accessible, and because the measurement likelihood $p_{\mathbf y \mid 0}$ can be tractability marginalized with respect to the prior $p_{\mathrm pr}^N$, the posterior density can be tractably computed at arbitrary points.

\begin{proposition}
      \label[proposition]{thm:finite-sample-posterior}
      Given measurement $\mathbf y$, the posterior $p_{t \mid \mathbf y}^N$ of noised samples $\mathbf x_t$ at time $t$ produced by the variance-preserving SDE defined in \Cref{eq:vpsde} acting on empirical distribution $p_{\mathrm{pr}}^N$ is a Gaussian mixture
      \begin{equation}
            p_{t \mid \mathbf y}^N(\mathbf x_t \mid \mathbf y) = \sum_{i = 1}^N \frac{\mathcal{N}(\mathbf y; \mathcal{A}(\mathbf x^{(i)}), \bm\Sigma_{\mathbf y})}{\sum_{j = 1}^N \mathcal{N}(\mathbf y; \mathcal{A}(\mathbf x^{(j)}), \bm\Sigma_{\mathbf y})} \mathcal{N}(\mathbf x_t; \sqrt{\bar\alpha(t)}\mathbf x^{(i)}, (1 - \bar\alpha(t))\mathbf I_n).
      \end{equation}
\end{proposition}
\begin{proof}
      See \Cref{proof:finite-sample-posterior}.
\end{proof}

\Cref{thm:finite-sample-posterior} gives us the ability to analytically approximate the posterior score in the finite-sample regime. More importantly, it gives approximate posterior density access. One consequence of this is the ability to directly draw samples from the approximate posterior at arbitrary positive time $t > 0$ without iterative (e.g. Euler-Maruyama) denoising, which is significantly cheaper numerically. 

We view the finite-sample objects as Monte Carlo approximations using the finitely-many i.i.d.\ samples from the unknown prior $p_{\mathrm{pr}}$, and thus expect to need on the order of $N^{-1/2}$ samples for our finite-sample method to be accurate. However, the constant factor depends heavily on time $t$, so while we need a relatively small dataset to accurately evaluate the posterior at intermediate times, we require a significantly larger dataset of prior samples as $t$ approaches zero, increasing both memory usage and computational cost. Furthermore, like the unconditional denoiser, the posterior devolves to a discrete measure at the original data at time zero, meaning that we must additionally modify our algorithm if our goal were to generalize outside of the original dataset.

Rather than attempting to use the finite-sample regime to produce new posterior samples $\mathbf x_t \sim p_{t \mid \mathbf y}$ for small time $t \approx 0$, where our method is relatively inaccurate even for large dataset size $N$, we use the finite-sample regime as a proxy for the ground truth posterior at intermediate timesteps, which allows us to explore the behavior of existing samplers on (i) priors where the denoiser is inaccessible and (ii) measurement models where the likelihood marginal does not simplify analytically.

\section{Experiments}
\label{sec:experiments}
Each of our experiments is defined by selecting (i) a prior distribution $p_{\mathrm{pr}}$, (ii) a measurement operator $\mathcal{A}$, and (iii) a noise covariance matrix $\bm\Sigma_{\mathbf y}$. When the measurement operator $\mathcal{A}(\mathbf x) = \mathbf A \mathbf x$ is linear, we demonstrate the performance of four moment-matching methods against the finite-sample regime:
\begin{enumerate}
    \item $\sigma$-DPS \citep{chung2023diffusion}, defined in \Cref{eq:chung-score} with unmodified prefactor $\sigma_{\mathbf y}^{-2}$,
    \item $\zeta$-DPS \citep{chung2023diffusion}, defined in \Cref{eq:modified-chung-score} with modified prefactor $\zeta/\|\mathbf y - \mathcal{A}(\mathbf m_{0 \mid t}(\mathbf x_t))\|_2$,
    \item $\Pi$GDM \citep{song2023pseudoinverseguided}, defined in \Cref{eq:song-approximation}, and 
    \item TMPD \citep{boys2024tweedie}, defined in \Cref{eq:boys-approximation}. 
\end{enumerate}
When $\mathcal{A}$ is nonlinear, we only study the two DPS methods against the finite-sample regime.

To assess the accuracy of our finite-sample posterior sampler, we begin with settings where the true posterior distribution $p_{t \mid \mathbf y}(\mathbf x_t \mid \mathbf y)$ is known analytically. We first consider two classes of test problems with multimodal priors: discrete priors with arbitrary measurement operators, and Gaussian mixture (or simply Gaussian) priors with affine measurement operators:
\begin{align}
    &p_{\mathrm{pr}} = \sum_i p_i \delta_{\mathbf x^{(i)}}, \quad \mathcal{A}(\mathbf x) \text{~general}, \tag{Discrete} \\
    &p_{\mathrm{pr}} = \mathcal{N}(\mathbf m_{\mathrm{pr}}, \mathbf C_{\mathrm{pr}}), \quad \mathcal{A}(\mathbf x) = \mathbf A\mathbf x + \mathbf b , \tag{Gaussian} \\
    &p_{\mathrm{pr}} = \sum_i p_i \mathcal{N}(\mathbf m_{\mathrm{pr}, i}, \mathbf C_{\mathrm{pr}, i}), \quad \mathcal{A}(\mathbf x) = \mathbf A\mathbf x + \mathbf b .\tag{GMM}
\end{align}
For these three settings, we demonstrate in Appendices \ref{sec:discrete-prior-analysis}, \ref{sec:gaussian-prior-analysis}, and \ref{sec:gmm-prior-analysis} respectively that the (unconditional) marginal score $\nabla_{\mathbf x_t} \log p_t(\mathbf x_t)$, the likelihood and likelihood score, and the posterior and posterior score are all analytically available. Once the accuracy of our approach is established, we use the finite-sample regime as a ground truth surrogate for settings where the true likelihood and posterior are not tractably computable, specifically Gaussian and GMM priors with nonlinear measurement operator $\mathcal{A}(\mathbf x)$.

For each experiment, we construct heatmaps (such as those in \Cref{fig:linear-multimodal}) for the ground truth by evaluating the analytically known density. For the finite-sample regime (FSR), we similarly evaluate the finite-sample posterior $p_{t \mid \mathbf y}^N$ (\Cref{thm:finite-sample-posterior}) on the same grid, where the dataset size $N$ is chosen empirically. For the four target baselines, we draw $K$ samples from the unit Gaussian at time $1$ and denoise using the analytically known unconditional marginal score for both the prior score contribution and for the approximated likelihood score via Tweedie's formula. For $\zeta$-DPS, the hyperparameter $\zeta$ is hand-chosen via grid search, as shown in \Cref{sec:zeta-tuning}.

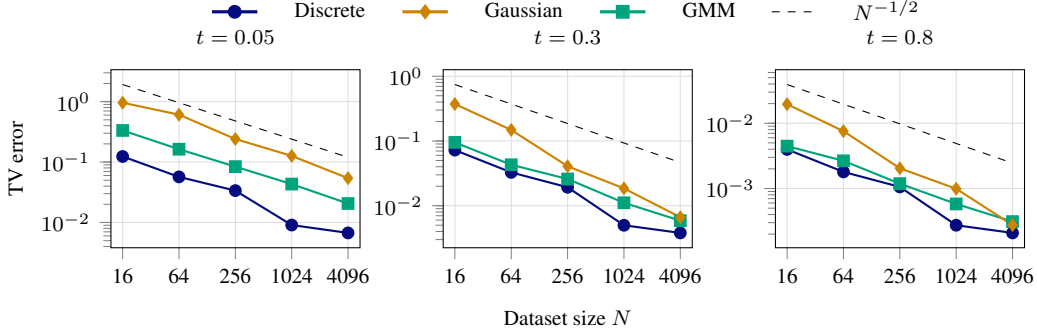
\begin{figure}[t]
    \centering
    \def\fsrTvPath{tikz-plots/fsr_tv/}%
\providecommand{\fsrTvPath}{}

\begingroup

\begin{tikzpicture}

\definecolor{grad1}{RGB}{4,   12,  128}
\definecolor{grad4}{RGB}{0,   164, 124}
\definecolor{grad7}{RGB}{204, 132, 0}

\pgfplotsset{
    panel/.style={
        width=1.92in, height=1.55in,
        xmode=log, ymode=log,
        log basis x={2},
        xtick={16, 64, 256, 1024, 4096},
        xticklabels={16, 64, 256, 1024, 4096},
        xmin=12, xmax=5500,
        grid=major, grid style={gray!25, thin},
        tick align=outside, tick pos=left,
        every axis/.append style={font=\footnotesize},
        every legend/.append style={font=\footnotesize},
        every node/.append style={font=\footnotesize},
    },
}

\begin{groupplot}[
    group style={
        group size=3 by 1,
        horizontal sep=1.1cm,
        ylabels at=edge left,
    },
    panel,
    ylabel={TV error},
]

\nextgroupplot[
    title={$t = 0.05$},
    legend to name=sharedleg, legend columns=4,
    legend style={draw=none, fill=none, column sep=0.4cm, font=\footnotesize},
]
\addlegendimage{grad1, very thick, mark=*} \addlegendentry{Discrete}
\addlegendimage{grad7, very thick, mark=diamond*} \addlegendentry{Gaussian}
\addlegendimage{grad4, very thick, mark=square*} \addlegendentry{GMM}
\addlegendimage{black, dashed, thin} \addlegendentry{$N^{-1/2}$}
\addplot[grad1, thick, mark=*]       table[col sep=comma, header=false, x index=0, y index=1] {\fsrTvPath pent_asym__quadratic__y=+1.09__tv.csv};
\addplot[grad4, thick, mark=square*] table[col sep=comma, header=false, x index=0, y index=1] {\fsrTvPath tri_equal__gain_highnoise__y=+1.00__tv.csv};
\addplot[grad7, thick, mark=diamond*] table[col sep=comma, header=false, x index=0, y index=1] {\fsrTvPath wide__gain_shift__y=-2.00__tv.csv};
\addplot[black, dashed, thin, domain=16:4096, samples=2, forget plot] {7.68/sqrt(x)};

\nextgroupplot[title={$t = 0.3$}]
\addplot[grad1, thick, mark=*]       table[col sep=comma, header=false, x index=0, y index=2] {\fsrTvPath pent_asym__quadratic__y=+1.09__tv.csv};
\addplot[grad4, thick, mark=square*] table[col sep=comma, header=false, x index=0, y index=2] {\fsrTvPath tri_equal__gain_highnoise__y=+1.00__tv.csv};
\addplot[grad7, thick, mark=diamond*] table[col sep=comma, header=false, x index=0, y index=2] {\fsrTvPath wide__gain_shift__y=-2.00__tv.csv};
\addplot[black, dashed, thin, domain=16:4096, samples=2, forget plot] {2.98/sqrt(x)};

\nextgroupplot[title={$t = 0.8$}]
\addplot[grad1, thick, mark=*]       table[col sep=comma, header=false, x index=0, y index=3] {\fsrTvPath pent_asym__quadratic__y=+1.09__tv.csv};
\addplot[grad4, thick, mark=square*] table[col sep=comma, header=false, x index=0, y index=3] {\fsrTvPath tri_equal__gain_highnoise__y=+1.00__tv.csv};
\addplot[grad7, thick, mark=diamond*] table[col sep=comma, header=false, x index=0, y index=3] {\fsrTvPath wide__gain_shift__y=-2.00__tv.csv};
\addplot[black, dashed, thin, domain=16:4096, samples=2, forget plot] {0.1568/sqrt(x)};

\end{groupplot}
\node[above=1.0em, font=\footnotesize] at ($(group c1r1.north)!0.5!(group c3r1.north)$) {\pgfplotslegendfromname{sharedleg}};
\node[below=2.0em, font=\footnotesize] at ($(group c1r1.south)!0.5!(group c3r1.south)$) {Dataset size $N$};
\end{tikzpicture}
\endgroup
    \vspace{-3mm}
    \caption{Each panel shows the median error of the finite-sample posterior sampler $p(\mathbf x_t \mid \mathbf y)$ in total variation (TV) for three example priors at a fixed diffusion time $t$. As expected, we see that error decreases in dataset size $N$ at the expected Monte Carlo rate for each fixed $t$, and that more data are needed to accurately capture the posterior at smaller times.}
    \label{fig:posterior-error-scaling}
\end{figure}
\begin{figure}[t]
    \centering
\begingroup
\graphicspath{{./}{tikz-plots/image_grid/}}

\begin{tikzpicture}

\pgfplotsset{
    cell/.style={
        width=0.78in, height=0.70in,
        scale only axis,
        enlargelimits=false,
        axis on top,
        xmin=0, xmax=1, ymin=0, ymax=1,
        xtick={0, 0.5, 1.0},
        ytick=\empty,
        tick align=outside, tick pos=left,
        every axis/.append style={font=\footnotesize},
        every node/.append style={font=\footnotesize},
    },
}

\begin{groupplot}[
    group style={
        group size=6 by 3,
        horizontal sep=0.25cm,
        vertical sep=0.25cm,
        xticklabels at=edge bottom,
    },
    cell,
]

\nextgroupplot[title={True posterior}]                  \addplot graphics[xmin=0,xmax=1,ymin=0,ymax=1]{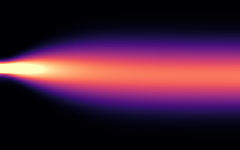};
\nextgroupplot[title={FSR (ours)}]                      \addplot graphics[xmin=0,xmax=1,ymin=0,ymax=1]{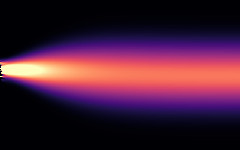};
\nextgroupplot[title={$\sigma$-DPS~\citep{chung2023diffusion}}]        \addplot graphics[xmin=0,xmax=1,ymin=0,ymax=1]{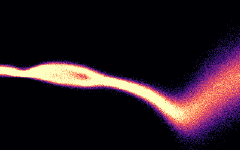};
\nextgroupplot[title={$\zeta$-DPS~\citep{chung2023diffusion}}] \addplot graphics[xmin=0,xmax=1,ymin=0,ymax=1]{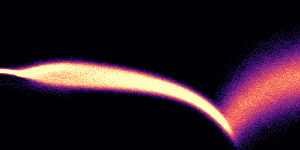};
\nextgroupplot[title={$\Pi$GDM~\citep{song2023pseudoinverseguided}}] \addplot graphics[xmin=0,xmax=1,ymin=0,ymax=1]{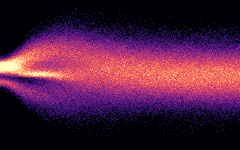};
\nextgroupplot[title={TMPD~\citep{boys2024tweedie}}]          \addplot graphics[xmin=0,xmax=1,ymin=0,ymax=1]{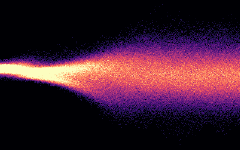};

\nextgroupplot \addplot graphics[xmin=0,xmax=1,ymin=0,ymax=1]{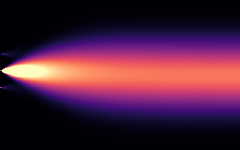};
\nextgroupplot \addplot graphics[xmin=0,xmax=1,ymin=0,ymax=1]{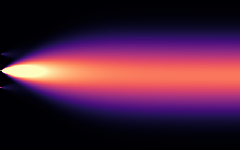};
\nextgroupplot \addplot graphics[xmin=0,xmax=1,ymin=0,ymax=1]{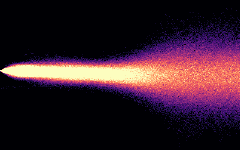};
\nextgroupplot \addplot graphics[xmin=0,xmax=1,ymin=0,ymax=1]{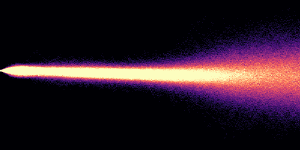};
\nextgroupplot \addplot graphics[xmin=0,xmax=1,ymin=0,ymax=1]{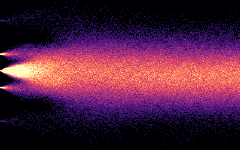};
\nextgroupplot \addplot graphics[xmin=0,xmax=1,ymin=0,ymax=1]{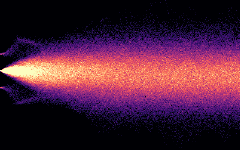};

\nextgroupplot \addplot graphics[xmin=0,xmax=1,ymin=0,ymax=1]{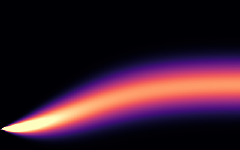};
\nextgroupplot \addplot graphics[xmin=0,xmax=1,ymin=0,ymax=1]{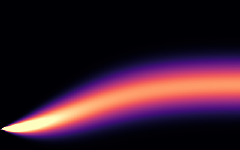};
\nextgroupplot \addplot graphics[xmin=0,xmax=1,ymin=0,ymax=1]{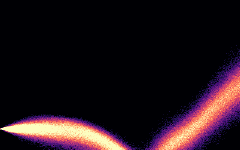};
\nextgroupplot \addplot graphics[xmin=0,xmax=1,ymin=0,ymax=1]{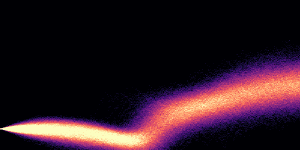};
\nextgroupplot \addplot graphics[xmin=0,xmax=1,ymin=0,ymax=1]{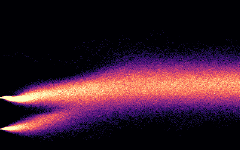};
\nextgroupplot \addplot graphics[xmin=0,xmax=1,ymin=0,ymax=1]{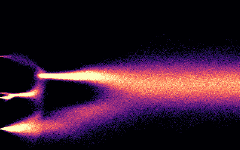};

\end{groupplot}

\node[below=1.4em, font=\footnotesize] at ($(group c1r3.south)!0.5!(group c6r3.south)$) {Time $t$};
\node[rotate=90, font=\footnotesize] at ([xshift=-1.8em] group c1r2.west) {Position $x$};

\end{tikzpicture}
\endgroup
    \vspace{-6mm}
    \caption{The reverse processes of the three moment-matching methods are shown against the ground truth and finite-sample posteriors on three test problems. We observe that the ground truth and finite-sample posteriors are qualitatively identical for sufficiently large $t$. The DPS Dirac approximation captures means accurately for small time, but has incorrect mean at intermediate time steps and incorrect variance for all time. Finally, the Gaussian approximations sample from prior modes which are inconsistent with the provided measurement.}
    \label{fig:linear-multimodal}
\end{figure}
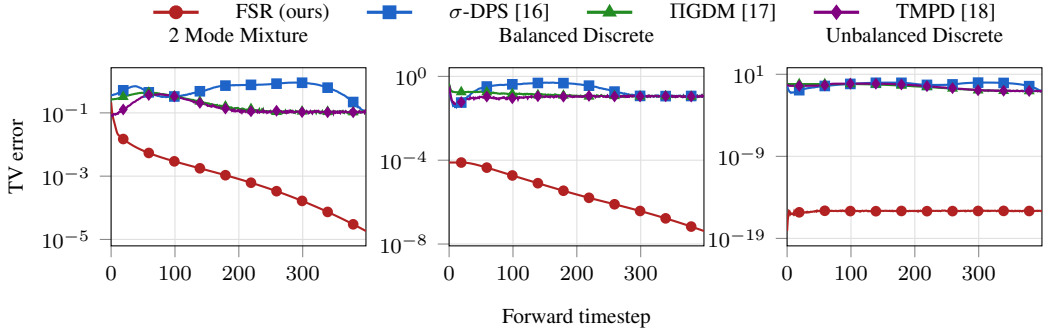
\begin{figure}
    \centering
    \def\imgridErrPath{tikz-plots/image_grid_error/}%
\providecommand{\imgridErrPath}{}

\begin{tikzpicture}

\definecolor{cFSR}{RGB}{178,  34,  34}
\definecolor{cDPS}{RGB}{ 30, 100, 200}
\definecolor{cPGDM}{RGB}{ 34, 139,  34}
\definecolor{cTMPD}{RGB}{128,   0, 128}

\pgfplotsset{
    panel/.style={
        width=1.95in, height=1.55in,
        ymode=log,
        xmin=0, xmax=399,
        xtick={0, 100, 200, 300, 400},
        grid=major, grid style={gray!25, thin},
        tick align=outside, tick pos=left,
        every axis/.append style={font=\footnotesize},
        every legend/.append style={font=\footnotesize},
        every node/.append style={font=\footnotesize},
    },
    methodplot/.style={thick, mark size=1.6pt, mark repeat=40, mark phase=20},
}

\begin{groupplot}[
    group style={
        group size=3 by 1,
        horizontal sep=1.1cm,
        ylabels at=edge left,
    },
    panel,
    ylabel={TV error},
]

\nextgroupplot[
    title={2 Mode Mixture},
    legend to name=sharedleg, legend columns=4,
    legend style={draw=none, fill=none, column sep=0.4cm, font=\footnotesize},
]
\addlegendimage{cFSR,  very thick, mark=*}         \addlegendentry{FSR (ours)}
\addlegendimage{cDPS,  very thick, mark=square*}   \addlegendentry{$\sigma$-DPS~\citep{chung2023diffusion}}
\addlegendimage{cPGDM, very thick, mark=triangle*} \addlegendentry{$\Pi$GDM~\citep{song2023pseudoinverseguided}}
\addlegendimage{cTMPD, very thick, mark=diamond*}  \addlegendentry{TMPD~\citep{boys2024tweedie}}
\addplot[cFSR,  methodplot, mark=*]         table[col sep=comma, header=false, x index=0, y index=1] {\imgridErrPath tv_bi.csv};
\addplot[cDPS,  methodplot, mark=square*]   table[col sep=comma, header=false, x index=0, y index=2] {\imgridErrPath tv_bi.csv};
\addplot[cPGDM, methodplot, mark=triangle*] table[col sep=comma, header=false, x index=0, y index=3] {\imgridErrPath tv_bi.csv};
\addplot[cTMPD, methodplot, mark=diamond*]  table[col sep=comma, header=false, x index=0, y index=4] {\imgridErrPath tv_bi.csv};

\nextgroupplot[title={Balanced Discrete}]
\addplot[cFSR,  methodplot, mark=*]         table[col sep=comma, header=false, x index=0, y index=1] {\imgridErrPath tv_pent.csv};
\addplot[cDPS,  methodplot, mark=square*]   table[col sep=comma, header=false, x index=0, y index=2] {\imgridErrPath tv_pent.csv};
\addplot[cPGDM, methodplot, mark=triangle*] table[col sep=comma, header=false, x index=0, y index=3] {\imgridErrPath tv_pent.csv};
\addplot[cTMPD, methodplot, mark=diamond*]  table[col sep=comma, header=false, x index=0, y index=4] {\imgridErrPath tv_pent.csv};

\nextgroupplot[title={Unbalanced Discrete}]
\addplot[cFSR,  methodplot, mark=*]         table[col sep=comma, header=false, x index=0, y index=1] {\imgridErrPath tv_wild.csv};
\addplot[cDPS,  methodplot, mark=square*]   table[col sep=comma, header=false, x index=0, y index=2] {\imgridErrPath tv_wild.csv};
\addplot[cPGDM, methodplot, mark=triangle*] table[col sep=comma, header=false, x index=0, y index=3] {\imgridErrPath tv_wild.csv};
\addplot[cTMPD, methodplot, mark=diamond*]  table[col sep=comma, header=false, x index=0, y index=4] {\imgridErrPath tv_wild.csv};

\end{groupplot}
\node[above=1.0em, font=\footnotesize] at ($(group c1r1.north)!0.5!(group c3r1.north)$) {\pgfplotslegendfromname{sharedleg}};
\node[below=2.0em, font=\footnotesize] at ($(group c1r1.south)!0.5!(group c3r1.south)$) {Forward timestep};
\end{tikzpicture}
    \caption{The total variation error of the finite-sample regime (FSR, $N = 4096$) and three moment-matching approaches are computed for the three examples shown in \Cref{fig:linear-multimodal}. The FSR's accuracy deteriorates for fixed $N$ as $t \searrow 0$, but remains accurate for moderately large $t$.}
    \label{fig:tv-error}
\end{figure}

First, we observe that the finite-sample posterior converges to the true posterior at the expected Monte Carlo rate for each fixed $t$, as shown in \Cref{fig:posterior-error-scaling}. Importantly, we see that the number of samples needed for accurate posterior access increases as we sample closer to time zero. Second, we qualitatively compare the true and finite-sample posteriors at all intermediate timesteps in \Cref{fig:linear-multimodal}, and observe that the two are characteristically identical. 

In each of the three examples in \Cref{fig:posterior-error-scaling} and \Cref{fig:linear-multimodal}, the prior distribution is multimodal, but the linear measurement operator and small noise scale results in an informative-enough measurement that the posterior is unimodal. However, we empirically observe that both second-moment approximations \cite{song2023pseudoinverseguided, boys2024tweedie} sample from multimodal posteriors by producing denoised points which are probable under the prior but which are inconsistent with the provided measurement. While in the context of learned scores one might attribute this behavior to prior score mismatch, we observe this behavior when using the \emph{true} unconditional marginal score. 
We attribute these tail samples to an overestimation of the posterior covariance, which drives the tail samples $\mathbf x_t$ far from the image of prior data $\sqrt{\bar\alpha(t)} \mathbf x_0$ which are consistent with the measurement (i.e. the corresponding weight $w_i(\mathbf x_t, t)$ is small). In this region of the reverse process, the prior score dominates sampling, effectively behaving like unconditional generation. Even when the sample recovers, as in the first $\Pi$GDM example, we see that the resulting posterior has an overestimated tail, corresponding to significant oversampling of unlikely events. We observe that the Dirac approximation explores regions of low prior and posterior probability at intermediate timesteps, which may prove problematic in cases where the prior score is learned by score matching and therefore uncontrolled directly in the low probability region.

Finally, we demonstrate that the performance of the finite-sample regime is accurate at intermediate times for reasonable $N$. \Cref{fig:tv-error} shows that total variation error of the three baselines at all intermediate timesteps of the (discretized) reverse process is significantly greater than that of the finite-sample regime, but that FSR's quality degrades near time zero when the true prior admits a density (as the total variation distance between a discrete measure and a density is always 1).

\Cref{fig:quadratic-grid} shows experiments with nonlinear forward models, for which we can only compare our finite-sample approach (FSR) against the Dirac approximation of \citet{chung2023diffusion}. The first two rows involve discrete prior distributions for which the true posterior can be computed analytically; the remaining rows involve Gaussian or Gaussian-mixture priors for which the true posterior is intractable but our finite-sample approach provides an accurate proxy for the ground truth posterior across positive timesteps.
We observe that $\sigma$-DPS regularly samples from prior modes that are not present in the true posterior (rows 2, 5, and 6) and/or fails to sample from modes present in the true posterior (rows 1, 3, and 4). Even in cases where $\sigma$-DPS captures the true modes of the posterior, it often weights them incorrectly.

Tuning the hyperparameter $\zeta$ in $\zeta$-DPS for each setting can improve the relative weighting of posterior modes (row 4) and reduce hallucination of likelihood-consistent modes that are not consistent with the prior (rows 5 and 6). However, tuning $\zeta$ cannot always avoid hallucinating prior-consistent modes that are not consistent with the measurement (row 2) nor avoid hallucinating measurement-consistent modes that are not plausible under the prior (row 4).

In general, by grid search over $\zeta$ (see \Cref{sec:zeta-tuning}) in our finite-sample lens we replicate the known behavior that higher values of $\zeta$ increase the weight of the measurement likelihood relative to the prior while smaller values of $\zeta$ increase the weight of the prior relative to the likelihood. Our finite-sample lens also reveals that in some cases mode collapse can arise from selecting $\zeta$ to be too large, in some cases even sampling from a single mode that is hallucinatory rather than consistent with the true posterior.

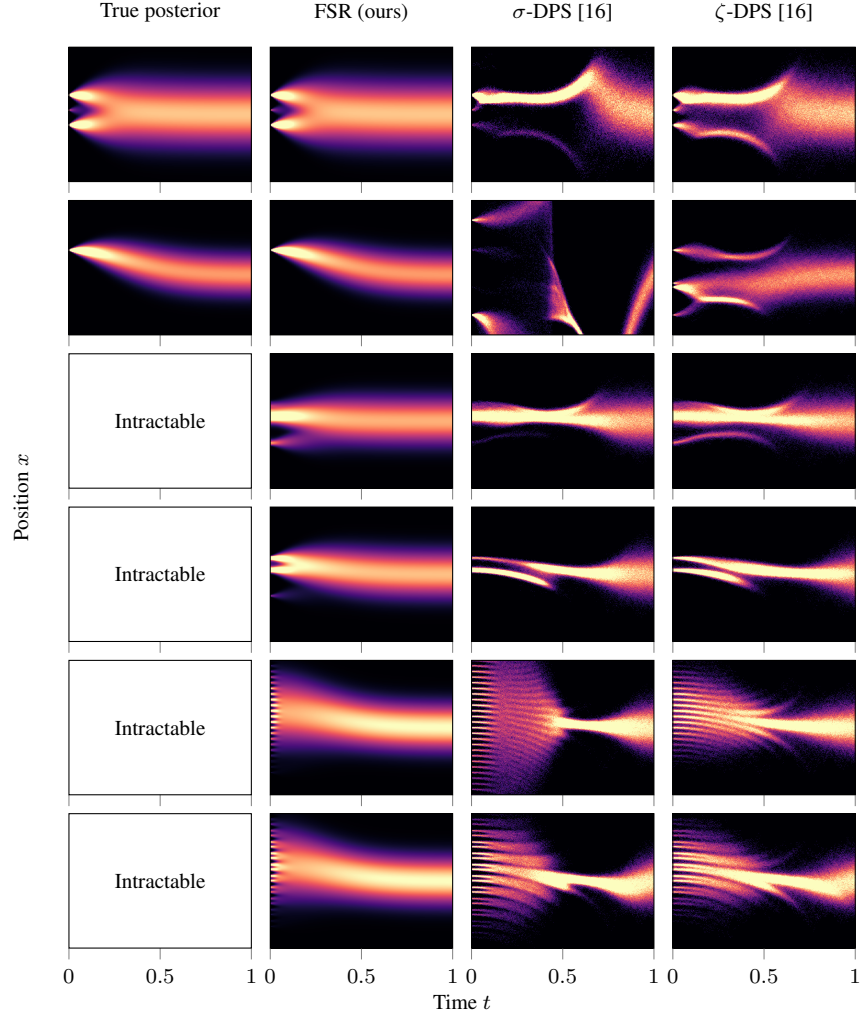
\begin{figure}
    \centering
\begingroup
\graphicspath{{./}{tikz-plots/quadratic_grid/}}

\begin{tikzpicture}

\pgfplotsset{
    cell/.style={
        width=0.95in, height=0.70in,
        scale only axis,
        enlargelimits=false,
        axis on top,
        xmin=0, xmax=1, ymin=0, ymax=1,
        xtick={0, 0.5, 1.0},
        ytick=\empty,
        tick align=outside, tick pos=left,
        every axis/.append style={font=\footnotesize},
        every node/.append style={font=\footnotesize},
    },
}

\begin{groupplot}[
    group style={
        group size=4 by 6,
        horizontal sep=0.25cm,
        vertical sep=0.25cm,
        xticklabels at=edge bottom,
    },
    cell,
]

\nextgroupplot[title={True posterior}]                  \addplot graphics[xmin=0,xmax=1,ymin=0,ymax=1]{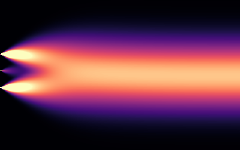};
\nextgroupplot[title={FSR (ours)}]                      \addplot graphics[xmin=0,xmax=1,ymin=0,ymax=1]{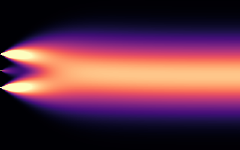};
\nextgroupplot[title={$\sigma$-DPS~\citep{chung2023diffusion}}]   \addplot graphics[xmin=0,xmax=1,ymin=0,ymax=1]{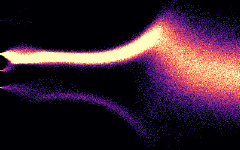};
\nextgroupplot[title={$\zeta$-DPS~\citep{chung2023diffusion}}]    \addplot graphics[xmin=0,xmax=1,ymin=0,ymax=1]{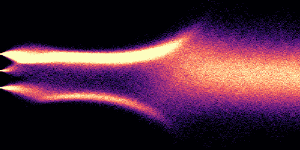};

\nextgroupplot \addplot graphics[xmin=0,xmax=1,ymin=0,ymax=1]{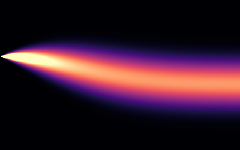};
\nextgroupplot \addplot graphics[xmin=0,xmax=1,ymin=0,ymax=1]{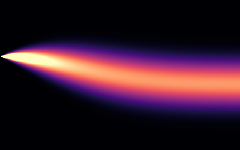};
\nextgroupplot \addplot graphics[xmin=0,xmax=1,ymin=0,ymax=1]{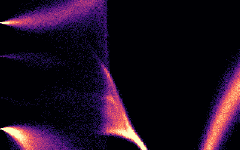};
\nextgroupplot \addplot graphics[xmin=0,xmax=1,ymin=0,ymax=1]{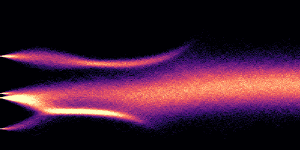};

\nextgroupplot \node[align=center] at (axis cs:0.5,0.5) {Intractable};
\nextgroupplot \addplot graphics[xmin=0,xmax=1,ymin=0,ymax=1]{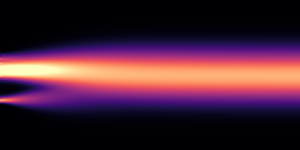};
\nextgroupplot \addplot graphics[xmin=0,xmax=1,ymin=0,ymax=1]{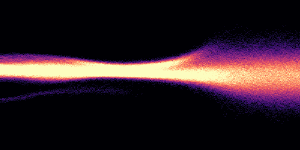};
\nextgroupplot \addplot graphics[xmin=0,xmax=1,ymin=0,ymax=1]{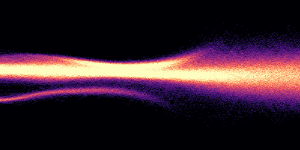};

\nextgroupplot \node[align=center] at (axis cs:0.5,0.5) {Intractable};
\nextgroupplot \addplot graphics[xmin=0,xmax=1,ymin=0,ymax=1]{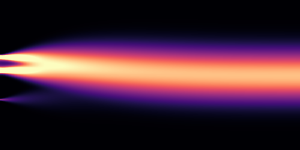};
\nextgroupplot \addplot graphics[xmin=0,xmax=1,ymin=0,ymax=1]{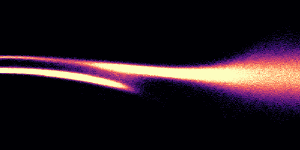};
\nextgroupplot \addplot graphics[xmin=0,xmax=1,ymin=0,ymax=1]{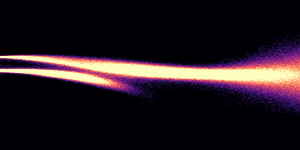};

\nextgroupplot \node[align=center] at (axis cs:0.5,0.5) {Intractable};
\nextgroupplot \addplot graphics[xmin=0,xmax=1,ymin=0,ymax=1]{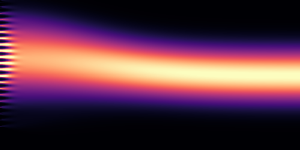};
\nextgroupplot \addplot graphics[xmin=0,xmax=1,ymin=0,ymax=1]{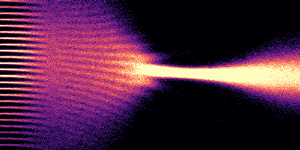};
\nextgroupplot \addplot graphics[xmin=0,xmax=1,ymin=0,ymax=1]{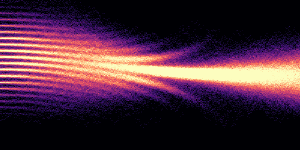};

\nextgroupplot \node[align=center] at (axis cs:0.5,0.5) {Intractable};
\nextgroupplot \addplot graphics[xmin=0,xmax=1,ymin=0,ymax=1]{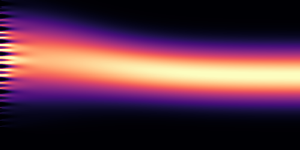};
\nextgroupplot \addplot graphics[xmin=0,xmax=1,ymin=0,ymax=1]{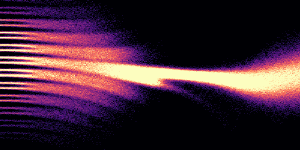};
\nextgroupplot \addplot graphics[xmin=0,xmax=1,ymin=0,ymax=1]{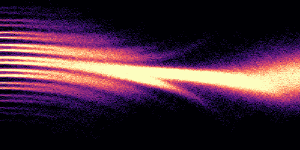};

\end{groupplot}

\node[below=1.4em, font=\footnotesize] at ($(group c1r6.south)!0.5!(group c4r6.south)$) {Time $t$};
\node[rotate=90, font=\footnotesize] at ($([xshift=-1.8em]group c1r3.west)!0.5!([xshift=-1.8em]group c1r4.west)$) {Position $x$};

\end{tikzpicture}
\endgroup
    \caption{For nonlinear measurement operators, the Dirac approximation generally demonstrates suboptimal posterior sampling, often failing to capture posterior modes and/or sampling from prior modes inconsistent with the provided measurement. For example, the true posterior in the second row collapses to a Dirac, but both DPS methods sample from prior models not present in the posterior, and underweight the lone posterior atom. Finally, we observe cases where minority modes are oversampled, such as with the middle mode in the first line.
    However, we see that proper tuning of $\zeta$ can result in improved performance over the $\sigma$-DPS baseline, such as by (partially) correcting the weighting between modes (first and fourth rows), by promoting sampling from modes which $\sigma$-DPS misses entirely (third row), or by discouraging sampling from modes which are consistent with the provided measurement but inconsitent with the prior (last two rows).
    \vspace{-4mm}
    }
    \label{fig:quadratic-grid}
\end{figure}

\section{Discussion}
\label{sec:discussion}
We introduce a finite-sample lens on posterior diffusion sampling, providing approximate access to the posterior distribution at all intermediate timesteps, for any prior and measurement model, and use this perspective to study how errors committed by existing methods arise during sampling. 

Although the common belief is that these moment-matching methods struggle with nonlinear measurement operators and/or multimodal posteriors, we show failure in cases where both the measurement operator is linear and the posterior is unimodal. We observe that the existing second-moment Gaussian approximations tend to overinflate the posterior covariance when the prior is multimodal, leading to hallucinations even when the true prior score is used. These inconsistent samples pose significant risk in real-world applications such as MRI: a model may reconstruct a tumor-free image from Fourier measurements of a tumor, but this failure may be difficult for a human to detect because the Fourier measurements are not (easily) visually interpretable. Conversely, in the nonlinear measurement operator case we observe that the first-moment Dirac approximation can entirely miss modes in the posterior, or get the relative weighting entirely wrong between posterior modes (\Cref{fig:quadratic-grid}).

Even when the posterior samples at time zero are of high quality (e.g. in contexts where existing benchmarking studies would deem a method to be accurate), the distribution at intermediate timesteps is often inaccurate. When intermediate distributions explore regions to which the (unconditional) marginal $p_t(\mathbf x_t)$ assigns low probability, the score matching objective (trained in expectation with respect to the marginal) is unlikely to have explicitly controlled the learned network in this region, leading to unpredictable behavior. Furthermore, this impacts downstream tasks which utilize intermediate states, such as out-of-distribution detection \citep{klip2026}.

A limitation of our method is that, without explicit computation of constant factors (as a function of $\bar\alpha(t)$, $t$, the prior, and $N$), it is difficult to reason about how large of a dataset is required to capture the posterior pointwise at time $t$ to sufficient accuracy.

Studying the impact of prior learning on sample quality is the most obvious extension of our work, including cases where (i) the prior under- or over-emphasizes minority modes (due to sensitivity of under- or over-representation of rare events in the training set), and (ii) the posterior sampler enters regions of low unconditional marginal probability. We conjecture that issues arising from particles traveling away from measurement-consistent regions and into prior-consistent regions of the reverse process will amplify when the prior is learned.

\begin{ack}
  BAB is supported by the National Science Foundation Graduate Research Fellowship Program under Grant No. DGE-2039655. Any opinions, findings, and conclusions or recommendations expressed in this material are those of the author(s) and do not necessarily reflect the views of the National Science Foundation. BAB would like to knowledge Sebasti\'an Guti\'errez Hern\'andez for helpful discussions.
\end{ack}

\medskip

{
\small

\bibliography{zotero, manual}



}

\clearpage
\appendix
\addtocontents{toc}{\setcounter{tocdepth}{2}}
\tableofcontents

\section{Related work}
\label[appendix]{sec:related-work}
\subsection{Generative posterior samplers}
\label[appendix]{sec:generative-posterior-samplers}

The goal of our study is to analyze likelihood (and likelihood score) approximations which assume access to a pretrained prior score, but no access to the forward model at training time.

Because the forward model is assumed to be inaccessible offline, methods which employ additional learning of the likelihood or posterior scores \emph{after} the forward model has been fixed \citep{elata2025psc,batzolis2021conditional,dhariwal2021diffusion,karras2022elucidating,denker2024deft} are beyond our scope. We additionally do not consider variable splitting methods \citep{xu2024provably,wu2024principled}, which often perform exact likelihood sampling (e.g., via Langevin dynamics or MALA). However, because moment-matching approximations can be (and are) used as initializers for more expensive and accurate methods \citep{xu2024provably}, studies of these moment-matching methods inform us of the diversity and accuracy of initializations that downstream algorithms can expect to build on.

There are alternative, flow-driven generative solutions to posterior sampling, such as through use of stochastic interpolants \citep{albergo2025stochastic, albergo2024stochastic} and flow-matching \citep{lipman2023flow,kim2025flowdps}. Due to the significant differences in how each class of method performs sampling, we do not seek a general framework categorizing when, how, and why \emph{all} generative posterior samplers fail. Although the zero-shot framework has been investigated in the flow matching context \citep{kim2025flowdps}, we limit our study to diffusion models.

\subsection{Benchmarking vs. understanding}
\label[appendix]{sec:benchmarking-vs-understanding}

We emphasize the difference between benchmarking studies and our goal of \emph{understanding}. Benchmarking studies typically draw posterior samples $p(\mathbf x_0 \mid \mathbf y)$ using a variety of algorithms, compare the quality of either individual samples (e.g. MAP quality) or the quality of the ensemble (e.g. posterior coverage), and identify which algorithms struggle in which contexts. However, benchmarking the end distribution alone is insufficient for concretely explaining \emph{why} the observed behavior arises.

For example, it is known \citep{zheng2025inversebench, zhang2025improving} that moment-matching methods, in particular the first-moment Dirac method \citep{chung2023diffusion}, struggle when the true posterior is multimodal, such as when the posterior support is non-convex \citep{zheng2025inversebench}. However, it is not clear (i) when in sampling this issue arises (e.g., at what diffusion time the moment-matching methods diverge from the desired behavior), or (ii) how this can be corrected. \citet{zheng2025inversebench} note that this multimodal behavior is not extensively discussed in the literature.

Building on observational knowledge from prior work in benchmarking that identifies which algorithms struggle in which contexts, our goal is to study how and why this error arises in specific examples, and draw general conclusions that inform practitioners on which algorithms can be used in which settings beyond existing benchmarks.

\subsection{Exact diffusion}
\label[appendix]{sec:exact-diffusion}

The finite-sample regime for diffusion models has been previously studied in the context of unconditional sampling. For example, \citet{scarvelis2025closedform} employ the finite-sample regime to develop a sampling method, and describe how to smooth the finite-sample method to promote generalization. \citet{mimikos-stamatopoulos2024scorebased} study various sources of error in diffusion sampling, including error arising from operating with finite samples.

\citet{zhang2025exacta} employ the finite-sample regime to perform posterior sampling by assuming the prior is a \emph{Gaussian mixture} with one component per training datum, but only study the case of linear forward models. \citet{wang2026errora} analyze the error of the resulting method. \citet{zhang2025iensf} employ the Gaussian mixture diffusion model for the analysis step in data assimilation. Although the method supports nonlinear measurement operators, the measurement operator is linearized in practice. Most importantly, all three works use the finite-sample perspective to develop posterior samplers, not to study existing samplers (as is our goal).

\section{Useful lemmas}
\label[appendix]{sec:useful-lemmas}

For the remainder of this section, we denote vectors $\mathbf x \in \mathcal{X} \subseteq \mathbb{R}^n$ and $\mathbf y \in \mathcal{Y} \subseteq \mathbb{R}^m$, matrix $\mathbf A \in \mathbb{R}^{m \times n}$, (possibly nonlinear) map $\mathcal{A}\colon \mathcal{X} \to \mathcal{Y}$, and symmetric positive definite matrices $\mathbf C_{\mathrm{pr}} \in \mathbb{R}^{n \times n}$ and $\bm\Sigma_{\mathbf y} \in \mathbb{R}^{m \times m}$.

\subsection{Linear updates of discrete measures}
\label[appendix]{sec:linear-updates-discrete}

\begin{lemma}[Gaussian-discrete marginal]
      \label[lemma]{lem:gaussian-discrete-marginal}
      Given Gaussian likelihood and discrete measure prior
      \begin{equation*}
            p(\mathbf y\mid \mathbf x) = \mathcal{N}(\mathbf y; \mathbf A \mathbf x, \bm\Sigma_{\mathbf y}), \quad \mu = \sum_{i = 1}^{N} w_i \delta_{\mathbf x^{(i)}},
      \end{equation*}
      with weights $w_i \geq 0$ and $\sum w_i = 1$, the marginal distribution $p(\mathbf y)$ is a Gaussian mixture
      \begin{equation*}
            p(\mathbf y) = \sum_{i = 1}^{N} w_i \, \mathcal{N}(\mathbf y; \mathbf A \mathbf x^{(i)}, \bm\Sigma_{\mathbf y}).
      \end{equation*}
\end{lemma}

\begin{proof}
First, write the marginal of interest $p(\mathbf y)$ in integral form.
\begin{align}
      p(\mathbf y)
      &= \int_{\mathcal{X}} p(\mathbf y\mid \mathbf x) \mu(\mathrm{d}\mathbf x), \\
      &= \int_{\mathcal{X}} \mathcal{N}(\mathbf y; \mathbf A \mathbf x, \bm\Sigma_{\mathbf y}) \mu(\mathrm{d}\mathbf x).
      \label{eq:discrete-marginal}
\end{align}
Because $\mu$ is a discrete measure, we have for any measurable $g : \mathcal{X} \to \mathbb{R}$ and measurable set $A$
\begin{equation}
      \int_A g(\mathbf x) \mu(\mathrm{d}\mathbf x) = \sum_{i = 1}^N w_i g(\mathbf x^{(i)}) \chi_A(\mathbf x^{(i)}).
      \label{eq:atomic-measure-integral}
\end{equation}
Applying \Cref{eq:atomic-measure-integral} to \Cref{eq:discrete-marginal} with $A = \mathbb{R}^n$ gives the desired result.
\begin{align}
      p(\mathbf y)
      = \sum_{i = 1}^{N} w_i\, \mathcal{N}(\mathbf y; \mathbf A \mathbf x^{(i)}, \bm\Sigma_{\mathbf y}).
\end{align}
\end{proof}

\begin{lemma}[Gaussian-discrete posterior]
      \label[lemma]{lem:gaussian-discrete-posterior}
      Given Gaussian likelihood and discrete measure prior
      \begin{align*}
            &p(\mathbf y\mid \mathbf x) = \mathcal{N}(\mathbf y; \mathbf A \mathbf x, \bm\Sigma_{\mathbf y}), \quad \mu = \sum_{i = 1}^{N} w_i \delta_{\mathbf x^{(i)}}
      \end{align*}
      where $w_i \geq 0$ and $\sum w_i = 1$, the posterior measure $\mu^y$ is a discrete measure with distribution and weights
      \begin{equation*}
            \mu^y = \sum_{i = 1}^{N} \tilde w_i(\mathbf y) \delta_{\mathbf x^{(i)}}, \quad
            \tilde w_i(\mathbf y) := \frac{w_i \,\mathcal{N}(\mathbf y; \mathbf A \mathbf x^{(i)}, \bm\Sigma_{\mathbf y})}{\sum_{j = 1}^{N} w_j \, \mathcal{N}(\mathbf y; \mathbf A \mathbf x^{(j)}, \bm\Sigma_{\mathbf y})}.
      \end{equation*}
\end{lemma}

\begin{proof}
By Bayes' rule, the definition of $p(\mathbf y\mid \mathbf x)$ by proposition, and $p(\mathbf y)$ by \Cref{lem:gaussian-discrete-marginal},
\begin{align}
      \frac{d\mu^y}{d\mu}(\mathbf x) &= \frac{p(\mathbf y\mid \mathbf x)}{p(\mathbf y)}
      = \frac{\mathcal{N}(\mathbf y; \mathbf A\mathbf x, \bm\Sigma_{\mathbf y})}{p(\mathbf y)}
      = \frac{\mathcal{N}(\mathbf y; \mathbf A\mathbf x, \bm\Sigma_{\mathbf y})}{\sum_{j = 1}^{N} w_j \, \mathcal{N}(\mathbf y; \mathbf A\mathbf x^{(j)}, \bm\Sigma_{\mathbf y})}.
      \label{eq:discrete-bayes}
\end{align}
Applying Radon-Nikodym theorem for measurable set $A$, substituting \Cref{eq:discrete-bayes} and applying \Cref{eq:atomic-measure-integral}, the probability assigned to $A$ by $\mu^y$ is given by
\begin{align}
      \mu^y(A)
      &= \int_A \frac{d\mu^y}{d\mu}(\mathbf x) \mu(\mathrm{d}\mathbf x), \\
      &= \int_A \frac{\mathcal{N}(\mathbf y; \mathbf A\mathbf x, \bm\Sigma_{\mathbf y})}{\sum_{j = 1}^{N} w_j \, \mathcal{N}(\mathbf y; \mathbf A\mathbf x^{(j)}, \bm\Sigma_{\mathbf y})} \mu(\mathrm{d}\mathbf x), \\
      &= \sum_{i = 1}^N \frac{w_i \mathcal{N}(\mathbf y; \mathbf A\mathbf x^{(i)}, \bm\Sigma_{\mathbf y})}{\sum_{j = 1}^{N} w_j \, \mathcal{N}(\mathbf y; \mathbf A\mathbf x^{(j)}, \bm\Sigma_{\mathbf y})} \chi_A(\mathbf x^{(i)}).
\end{align}
Thus, as measures the posterior is a reweighting of the prior
\begin{equation}
      \mu^y = \sum_{i = 1}^N \tilde w_i(\mathbf y) \delta_{\mathbf x^{(i)}}, \qquad
      \tilde w_i(\mathbf y) := \frac{w_i \,\mathcal{N}(\mathbf y; \mathbf A \mathbf x^{(i)}, \bm\Sigma_{\mathbf y})}{\sum_{j = 1}^{N} w_j \, \mathcal{N}(\mathbf y; \mathbf A \mathbf x^{(j)}, \bm\Sigma_{\mathbf y})},
\end{equation}
whose weights are nonnegative and sum to 1,
\begin{equation}
      \sum_{i = 1}^N \tilde w_i(\mathbf y) = \frac{\sum_{i = 1}^N w_i \,\mathcal{N}(\mathbf y; \mathbf A \mathbf x^{(i)}, \bm\Sigma_{\mathbf y})}{\sum_{j = 1}^{N} w_j \, \mathcal{N}(\mathbf y; \mathbf A \mathbf x^{(j)}, \bm\Sigma_{\mathbf y})} = 1.
\end{equation}
\end{proof}

\subsection{Linear updates of Gaussians}
\label[appendix]{sec:linear-updates-gaussians}

\begin{lemma}[Gaussian-Gaussian marginal]
      \label[lemma]{lem:gaussian-gaussian-marginal}
      Given Gaussian likelihood and Gaussian prior
      \begin{equation*}
            p(\mathbf y\mid \mathbf x) = \mathcal{N}(\mathbf y; \mathbf A\mathbf x, \bm\Sigma_{\mathbf y}), \quad p(\mathbf x) = \mathcal{N}(\mathbf x; \mathbf m_{\mathrm{pr}}, \mathbf C_{\mathrm{pr}})
      \end{equation*}
      the marginal distribution $p(\mathbf y)$ is a Gaussian with distribution
      \begin{equation*}
            p(\mathbf y) = \mathcal{N}(\mathbf y; \mathbf A \mathbf m_{\mathrm{pr}}, \bm\Sigma_{\mathbf y} + \mathbf A \mathbf C_{\mathrm{pr}}\mathbf A^\top).
      \end{equation*}
\end{lemma}

\begin{proof}
By reparametrization, we can express the likelihood $p(\mathbf y \mid \mathbf x)$ as
\begin{equation}
      \mathbf y = \mathbf A \mathbf x + \bm\eta, \quad \bm\eta \sim \mathcal{N}(\mathbf 0, \bm\Sigma_{\mathbf y}),
\end{equation}
where $\bm\eta$ is drawn independently of $\mathbf x$. Observe that $\mathbf A \mathbf x$ is a linear transformation of a Gaussian, which is a Gaussian with distribution
\begin{equation}
      \mathbf A \mathbf x \sim \mathcal{N}(\mathbf A \mathbf m_{\mathrm{pr}}, \mathbf A \mathbf C_{\mathrm{pr}} \mathbf A^\top).
\end{equation}
Thus, $\mathbf A \mathbf x + \bm\eta$ is a sum of independent Gaussians, which hence has distribution
\begin{equation}
      \mathbf y = \mathbf A \mathbf x + \bm\eta \sim \mathcal{N}(\mathbf y; \mathbf A \mathbf m_{\mathrm{pr}}, \bm\Sigma_{\mathbf y} + \mathbf A \mathbf C_{\mathrm{pr}}\mathbf A^\top).
\end{equation}
\end{proof}

\begin{lemma}[Gaussian-Gaussian conjugacy]
      \label[lemma]{lem:gaussian-gaussian-conjugacy}
      Given Gaussian likelihood and Gaussian prior
      \begin{equation*}
            p(\mathbf y\mid \mathbf x) = \mathcal{N}(\mathbf y; \mathbf A \mathbf x, \bm\Sigma_{\mathbf y}), \quad p(\mathbf x) = \mathcal{N}(\mathbf x; \mathbf m_{\mathrm{pr}}, \mathbf C_{\mathrm{pr}}),
      \end{equation*}
      the posterior distribution $p(\mathbf x \mid \mathbf y)$ is also Gaussian $\mathcal{N}\left(\mathbf x; \mathbf m_{\mathrm{post}}, \mathbf C_{\mathrm{post}}\right)$ with 
      covariance and mean
      \begin{align*}
            \mathbf C_{\mathrm{post}} &:= (\mathbf A^\top \bm\Sigma_{\mathbf y}^{-1}\mathbf A + \mathbf C_{\mathrm{pr}}^{-1})^{-1}, \\
            \mathbf m_{\mathrm{post}} &:= \mathbf C_{\mathrm{post}}(\mathbf A^\top\bm\Sigma_{\mathbf y}^{-1}\mathbf y + \mathbf C_{\mathrm{pr}}^{-1}\mathbf m_{\mathrm{pr}}).
      \end{align*}
\end{lemma}

\begin{proof}
We use $\textrm{const.}$ to denote a collection of terms which are constants or are constants with respect to the input $\mathbf x$, such as normalization constants that are functions of $\mathbf y$.
By Bayes' rule, the posterior is given by
\begin{equation}
      p(\mathbf x \mid \mathbf y) = \frac{p(\mathbf y \mid \mathbf x)p(\mathbf x)}{p(\mathbf y)}.
\end{equation}
Taking logarithms,
\begin{equation}
      \log p(\mathbf x \mid \mathbf y) = \log p(\mathbf y \mid \mathbf x) + \log p(\mathbf x) - \log p(\mathbf y),
\end{equation}
we can expand the log-likelihood and log-prior as
\begin{align}
      \log p(\mathbf y \mid \mathbf x) &= - \frac12 (\mathbf y - \mathbf A \mathbf x)^\top \bm\Sigma_{\mathbf y}^{-1}(\mathbf y - \mathbf A \mathbf x) - \frac12 \log((2\pi)^{m} \det(\bm\Sigma_{\mathbf y})),\label{eq:gaussian-log-likelihood}\\
      \log p(\mathbf x) &= - \frac12 (\mathbf x - \mathbf m_{\mathrm{pr}})^\top \mathbf C_{\mathrm{pr}}^{-1}(\mathbf x - \mathbf m_{\mathrm{pr}}) - \frac12 \log((2\pi)^{n} \det(\mathbf C_{\mathrm{pr}}))\label{eq:gaussian-log-prior}.
\end{align}
Observe that $\log p(\mathbf y)$ and the log-determinant terms in \Cref{eq:gaussian-log-likelihood} and \Cref{eq:gaussian-log-prior} are constant w.r.t.\ $\mathbf x$, thus we can re-express the log-posterior as
\begin{equation}
      - 2\log p(\mathbf x \mid \mathbf y) = (\mathbf y - \mathbf A \mathbf x)^\top \bm\Sigma_{\mathbf y}^{-1}(\mathbf y - \mathbf A \mathbf x) + (\mathbf x - \mathbf m_{\mathrm{pr}})^\top \mathbf C_{\mathrm{pr}}^{-1}(\mathbf x - \mathbf m_{\mathrm{pr}}) + \textrm{const.}
\end{equation}
Expanding the quadratic forms, combining like terms, and collapsing terms constant with respect to $\mathbf x$,
\begin{align}
      - 2\log p(\mathbf x \mid \mathbf y)
      &= {\color{black}(\mathbf y - \mathbf A \mathbf x)^\top \bm\Sigma_{\mathbf y}^{-1}(\mathbf y - \mathbf A \mathbf x)} + {\color{black}(\mathbf x - \mathbf m_{\mathrm{pr}})^\top \mathbf C_{\mathrm{pr}}^{-1}(\mathbf x - \mathbf m_{\mathrm{pr}})} + \textrm{const.}, \\
      &= {\color{black} \underbrace{\mathbf y^\top\bm\Sigma_{\mathbf y}^{-1}\mathbf y}_{\textrm{const. w.r.t. }\mathbf x} - 2 \mathbf x^\top \mathbf A^\top \bm\Sigma_{\mathbf y}^{-1}\mathbf y + \mathbf x^\top \mathbf A^\top \bm\Sigma_{\mathbf y}^{-1}\mathbf A \mathbf x} \\
      &\quad+ {\color{black}\mathbf x^\top\mathbf C_{\mathrm{pr}}^{-1}\mathbf x - 2\mathbf x^\top \mathbf C_{\mathrm{pr}}^{-1}\mathbf m_{\mathrm{pr}} + \underbrace{\mathbf m_{\mathrm{pr}}^\top \mathbf C_{\mathrm{pr}}^{-1} \mathbf m_{\mathrm{pr}}}_{\textrm{const. w.r.t. }\mathbf x}} + \textrm{const.} \notag \\
      &= {\color{black}- 2\mathbf x^\top \mathbf A^\top\bm\Sigma_{\mathbf y}^{-1}\mathbf y + \mathbf x^\top \mathbf A^\top \bm\Sigma_{\mathbf y}^{-1}\mathbf A\mathbf x} + {\color{black}\mathbf x^\top\mathbf C_{\mathrm{pr}}^{-1}\mathbf x - 2\mathbf x^\top \mathbf C_{\mathrm{pr}}^{-1}\mathbf m_{\mathrm{pr}}} + \textrm{const.} \\
      &= \mathbf x^\top (\mathbf A^\top\bm\Sigma_{\mathbf y}^{-1}\mathbf A + \mathbf C_{\mathrm{pr}}^{-1})\mathbf x - 2\mathbf x^\top(\mathbf A^\top\bm\Sigma_{\mathbf y}^{-1}\mathbf y + \mathbf C_{\mathrm{pr}}^{-1}\mathbf m_{\mathrm{pr}}) + \textrm{const.} \\
      &= \mathbf x^\top \mathbf C_{\mathrm{post}}^{-1}\mathbf x - 2\mathbf x^\top \mathbf b + \textrm{const.}\label{eq:intermediate-post}
\end{align}
where
\begin{equation}
      \mathbf C_{\mathrm{post}} := (\mathbf A^\top\bm\Sigma_{\mathbf y}^{-1}\mathbf A + \mathbf C_{\mathrm{pr}}^{-1})^{-1}, \quad \mathbf b := \mathbf A^\top\bm\Sigma_{\mathbf y}^{-1}\mathbf y + \mathbf C_{\mathrm{pr}}^{-1}\mathbf m_{\mathrm{pr}}.
\end{equation}
For vectors $\mathbf u, \mathbf v$ and symmetric matrix $\mathbf M$, we have the quadratic form identity,
\begin{equation}
      (\mathbf u - \mathbf v)^\top \mathbf M(\mathbf u - \mathbf v) = \mathbf u^\top \mathbf M \mathbf u - 2\mathbf u^\top \mathbf M\mathbf v + \mathbf v^\top \mathbf M \mathbf v,
      \label{eq:quadratic-form-identity}
\end{equation}
where the third term is notably constant with respect to $\mathbf u$. Thus, if we define $\mathbf m_{\mathrm{post}}$ such that
\begin{equation}
      \mathbf C_{\mathrm{post}}^{-1}\mathbf m_{\mathrm{post}} = \mathbf b \implies \mathbf m_{\mathrm{post}} := \left(\mathbf A^\top\bm\Sigma_{\mathbf y}^{-1}\mathbf A + \mathbf C_{\mathrm{pr}}^{-1}\right)^{-1}(\mathbf A^\top\bm\Sigma_{\mathbf y}^{-1}\mathbf y + \mathbf C_{\mathrm{pr}}^{-1}\mathbf m_{\mathrm{pr}}),
      \label{eq:def-mu-post}
\end{equation}
which is again constant w.r.t.\ $\mathbf x$, \Cref{eq:intermediate-post} can be re-expressed by completing the square as
\begin{align}
      - 2 \log p(\mathbf x \mid \mathbf y)
      &= \mathbf x^\top \mathbf C_{\mathrm{post}}^{-1}\mathbf x - 2\mathbf x^\top \mathbf b + \textrm{const.}, \\
      &= \mathbf x^\top \mathbf C_{\mathrm{post}}^{-1}\mathbf x - 2\mathbf x^\top \mathbf C_{\mathrm{post}}^{-1}\mathbf m_{\mathrm{post}} + \textrm{const.}, \\
      &= \mathbf x^\top \mathbf C_{\mathrm{post}}^{-1}\mathbf x - 2\mathbf x^\top \mathbf C_{\mathrm{post}}^{-1}\mathbf m_{\mathrm{post}} +  \mathbf m_{\mathrm{post}}^\top\mathbf C_{\mathrm{post}}^{-1}\mathbf m_{\mathrm{post}} \\
            &\quad- \mathbf m_{\mathrm{post}}^\top\mathbf C_{\mathrm{post}}^{-1}\mathbf m_{\mathrm{post}} + \textrm{const.}, \notag\\
      &= (\mathbf x - \mathbf m_{\mathrm{post}})^\top \mathbf C_{\mathrm{post}}^{-1}(\mathbf x - \mathbf m_{\mathrm{post}}) - \mathbf m_{\mathrm{post}}^\top\mathbf C_{\mathrm{post}}^{-1}\mathbf m_{\mathrm{post}} + \textrm{const.}, \\
      &= (\mathbf x - \mathbf m_{\mathrm{post}})^\top \mathbf C_{\mathrm{post}}^{-1}(\mathbf x - \mathbf m_{\mathrm{post}}) + \textrm{const.} \\
      \implies \log p(\mathbf x \mid \mathbf y) &= -\frac12 (\mathbf x - \mathbf m_{\mathrm{post}})^\top \mathbf C_{\mathrm{post}}^{-1}(\mathbf x - \mathbf m_{\mathrm{post}}) + \textrm{const.}
\end{align}
Observe that the right-hand side is the logarithm of a Gaussian density (with suitable normalization constant), thus the posterior is a Gaussian with mean and covariance
\begin{equation}
      \mathbf m_{\mathrm{post}} := \underbrace{\left(\mathbf A^\top\bm\Sigma_{\mathbf y}^{-1}\mathbf A + \mathbf C_{\mathrm{pr}}^{-1}\right)^{-1}}_{\mathbf C_{\mathrm{post}}}(\mathbf A^\top\bm\Sigma_{\mathbf y}^{-1}\mathbf y + \mathbf C_{\mathrm{pr}}^{-1}\mathbf m_{\mathrm{pr}}), \quad \mathbf C_{\mathrm{post}} = (\mathbf A^\top\bm\Sigma_{\mathbf y}^{-1}\mathbf A + \mathbf C_{\mathrm{pr}}^{-1})^{-1}.
\end{equation}
\end{proof}

\subsection{Linear updates of Gaussian Mixtures}
\label{eq:linear-updates-gmms}

\begin{lemma}[Gaussian-GMM marginal]
      \label[lemma]{lem:gaussian-gmm-marginal}
      Given Gaussian likelihood and Gaussian mixture prior
      \begin{equation*}
            p(\mathbf y\mid \mathbf x) = \mathcal{N}(\mathbf y; \mathbf A \mathbf x, \bm\Sigma_{\mathbf y}), \quad p(\mathbf x) = \sum_{i = 1}^{N} w_i\, \mathcal{N}(\mathbf x; \mathbf m_{\mathrm{pr}, i}, \mathbf C_{\mathrm{pr}, i}),
      \end{equation*}
      with weights $w_i \geq 0$ and $\sum w_i = 1$, the marginal $p(\mathbf y)$ is a Gaussian mixture
      \begin{equation*}
            p(\mathbf y) = \sum_{i = 1}^{N} w_i\,\mathcal{N}(\mathbf y; \mathbf A \mathbf m_{\mathrm{pr}, i}, \bm\Sigma_{\mathbf y} + \mathbf A\mathbf C_{\mathrm{pr}, i}\mathbf A^\top).
      \end{equation*}
\end{lemma}

\begin{proof}
First, write the marginal of interest $p(\mathbf y)$ in integral form.
\begin{equation}
      p(\mathbf y) = \int_{\mathcal{X}}p(\mathbf y \mid \mathbf x)p(\mathbf x)\mathrm{d}\mathbf x.
\end{equation}
Substituting the definitions for the likelihood and prior, distributing, and interchanging the integral and the finite sum, we obtain
\begin{align}
     p(\mathbf y)
     &= \int_{\mathcal{X}}p(\mathbf y \mid \mathbf x)p(\mathbf x)\mathrm{d}\mathbf x, \notag\\
     &= \int_{\mathcal{X}}\mathcal{N}(\mathbf y; \mathbf A \mathbf x, \bm\Sigma_{\mathbf y})\sum_{i = 1}^{N} w_i\, \mathcal{N}(\mathbf x; \mathbf m_{\mathrm{pr}, i}, \mathbf C_{\mathrm{pr}, i})\mathrm{d}\mathbf x, \\
     &= \int_{\mathcal{X}}\sum_{i = 1}^{N} w_i\,\mathcal{N}(\mathbf y; \mathbf A \mathbf x, \bm\Sigma_{\mathbf y}) \,\mathcal{N}(\mathbf x; \mathbf m_{\mathrm{pr}, i}, \mathbf C_{\mathrm{pr}, i})\mathrm{d}\mathbf x, \\
     &= \sum_{i = 1}^{N} w_i\int_{\mathcal{X}}\mathcal{N}(\mathbf y; \mathbf A \mathbf x, \bm\Sigma_{\mathbf y})\, \mathcal{N}(\mathbf x; \mathbf m_{\mathrm{pr}, i}, \mathbf C_{\mathrm{pr}, i})\mathrm{d}\mathbf x.
     \label{eq:marginal-mixture}
\end{align}
Observe the integral is a marginal distribution with a linear-Gaussian likelihood and Gaussian prior, and so by Lemma \ref{lem:gaussian-gaussian-marginal} can be re-written as
\begin{equation}
      \int_{\mathcal{X}}\mathcal{N}(\mathbf y; \mathbf A \mathbf x, \bm\Sigma_{\mathbf y})\, \mathcal{N}(\mathbf x; \mathbf m_{\mathrm{pr}, i}, \mathbf C_{\mathrm{pr}, i})\mathrm{d}\mathbf x = \mathcal{N}(\mathbf y; \mathbf A \mathbf m_{\mathrm{pr}, i}, \bm\Sigma_{\mathbf y} + \mathbf A\mathbf C_{\mathrm{pr}, i}\mathbf A^\top).
      \label{eq:rewritten-marginal}
\end{equation}
Substituting \Cref{eq:rewritten-marginal} into \Cref{eq:marginal-mixture} gives us the desired result.
\begin{equation}
      p(\mathbf y) = \sum_{i = 1}^{N} w_i\, \mathcal{N}(\mathbf y; \mathbf A \mathbf m_{\mathrm{pr}, i}, \bm\Sigma_{\mathbf y} + \mathbf A\mathbf C_{\mathrm{pr}, i}\mathbf A^\top).
\end{equation}
\end{proof}

\begin{lemma}[Gaussian-GMM conjugacy]
      \label[lemma]{lem:gaussian-gmm-conjugacy}
      Given a Gaussian likelihood and a Gaussian mixture prior
      \begin{equation*}
            p(\mathbf y\mid \mathbf x) = \mathcal{N}(\mathbf y; \mathbf A \mathbf x, \bm\Sigma_{\mathbf y}),
            \qquad
            p(\mathbf x) = \sum_{i = 1}^{N} w_i\, \mathcal{N}(\mathbf x; \mathbf m_{\mathrm{pr}, i}, \mathbf C_{\mathrm{pr}, i}),
      \end{equation*}
      with weights $w_i \geq 0$ and $\sum w_i = 1$, the posterior distribution $p(\mathbf x \mid \mathbf y)$ is a Gaussian mixture
      \begin{equation*}
            p(\mathbf x \mid \mathbf y) = \sum_{i = 1}^{N}\tilde w_{i}(\mathbf y) \mathcal{N}\left(\mathbf x; \mathbf m_{\mathrm{post}, i}, \mathbf C_{\mathrm{post}, i}\right),
      \end{equation*}
      with component covariances, means, and weights
      \begin{align*}
            \mathbf C_{\mathrm{post}, i}
            &:= \left(\mathbf A^\top\bm\Sigma_{\mathbf y}^{-1}\mathbf A + \mathbf C_{\mathrm{pr}, i}^{-1}\right)^{-1},\\
            \mathbf m_{\mathrm{post}, i} &:= \mathbf C_{\mathrm{post}, i}\left(\mathbf A^\top\bm\Sigma_{\mathbf y}^{-1}\mathbf y + \mathbf C_{\mathrm{pr}, i}^{-1}\mathbf m_{\mathrm{pr}, i}\right),\\
            \tilde w_i(\mathbf y)& := \frac{w_i\,\mathcal{N}\left(\mathbf y; \mathbf A \mathbf m_{\mathrm{pr}, i}, \bm\Sigma_{\mathbf y} + \mathbf A \mathbf C_{\mathrm{pr}, i}\mathbf A^\top\right)}{\sum_{j = 1}^{N}w_j\,\mathcal{N}\left(\mathbf y; \mathbf A \mathbf m_{\mathrm{pr}, j}, \bm\Sigma_{\mathbf y} + \mathbf A \mathbf C_{\mathrm{pr}, j}\mathbf A^\top\right)}.
      \end{align*}
\end{lemma}

\begin{proof}
By Bayes' rule, the posterior can be written as
\begin{equation}
p(\mathbf x \mid \mathbf y)
= \frac{p(\mathbf y \mid \mathbf x)p(\mathbf x)}{p(\mathbf y)}.
\end{equation}
For the denominator, observe that $p(\mathbf y)$ is a marginal with linear-Gaussian likelihood and Gaussian mixture prior, hence by \Cref{lem:gaussian-gmm-marginal} is
\begin{equation}
      p(\mathbf y) = \sum_{j = 1}^{N} w_j\,\mathcal{N}\left(\mathbf y; \mathbf A \mathbf m_{\mathrm{pr}, j}, \bm\Sigma_{\mathbf y} + \mathbf A \mathbf C_{\mathrm{pr}, j}\mathbf A^\top\right)
      \label{eq:final-denominator}
\end{equation}
For the numerator, we first substitute the known forms for the likelihood and prior to obtain
\begin{align}
p(\mathbf y \mid \mathbf x)p(\mathbf x)
&= \mathcal{N}(\mathbf y; \mathbf A \mathbf x, \bm\Sigma_{\mathbf y})
\sum_{i = 1}^{N} w_i\,\mathcal{N}(\mathbf x; \mathbf m_{\mathrm{pr}, i}, \mathbf C_{\mathrm{pr}, i}) \\
&= \sum_{i = 1}^{N} w_i\,\mathcal{N}(\mathbf y; \mathbf A \mathbf x, \bm\Sigma_{\mathbf y})\,\mathcal{N}(\mathbf x; \mathbf m_{\mathrm{pr}, i}, \mathbf C_{\mathrm{pr}, i}). \label{eq:posterior-numerator}
\end{align}
We can re-write the (unweighted) summand via Bayes' rule as
\begin{equation}
p_i(\mathbf x \mid \mathbf y) = \frac{\mathcal{N}(\mathbf y; \mathbf A \mathbf x, \bm\Sigma_{\mathbf y})\, \mathcal{N}(\mathbf x; \mathbf m_{\mathrm{pr}, i}, \mathbf C_{\mathrm{pr}, i})}{p_i(\mathbf y)}.
\label{eq:bayes-component-posterior}
\end{equation}
where $p_i(\mathbf y)$ and $p_i(\mathbf x \mid \mathbf y)$ are the marginal and posterior of the $i$th component. Observe that the posterior of the $i$th component is exactly the Gaussian-Gaussian conjugate form of \Cref{lem:gaussian-gaussian-conjugacy}, thus the corresponding posterior update of component $i$ is Gaussian with distribution
\begin{equation}
p_i(\mathbf x \mid \mathbf y)
= \mathcal{N}\left(\mathbf x; \mathbf m_{\mathrm{post}, i}, \mathbf C_{\mathrm{post}, i}\right),
\label{eq:component-posterior}
\end{equation}
with covariance and mean
\begin{align}
\mathbf C_{\mathrm{post}, i}
&= \left(\mathbf A^\top\bm\Sigma_{\mathbf y}^{-1}\mathbf A + \mathbf C_{\mathrm{pr}, i}^{-1}\right)^{-1}, \\
\mathbf m_{\mathrm{post}, i}
&= \mathbf C_{\mathrm{post}, i}
\left(\mathbf A^\top\bm\Sigma_{\mathbf y}^{-1}\mathbf y + \mathbf C_{\mathrm{pr}, i}^{-1}\mathbf m_{\mathrm{pr}, i}\right).
\end{align}
Similarly, the marginal of the $i$th component $p_i(\mathbf y)$ is exactly the Gaussian-Gaussian marginal form of Lemma~\ref{lem:gaussian-gaussian-marginal}, thus is a Gaussian with distribution
\begin{equation}
p_i(\mathbf y) = \mathcal{N}\left(\mathbf y; \mathbf A \mathbf m_{\mathrm{pr}, i}, \bm\Sigma_{\mathbf y} + \mathbf A \mathbf C_{\mathrm{pr}, i}\mathbf A^\top\right),
\label{eq:component-marginal}
\end{equation}
Rearranging \Cref{eq:bayes-component-posterior} and substituting \Cref{eq:component-posterior} and \Cref{eq:component-marginal},
\begin{align}
      \mathcal{N}(\mathbf y; \mathbf A \mathbf x, \bm\Sigma_{\mathbf y})\, \mathcal{N}(\mathbf x; \mathbf m_{\mathrm{pr}, i}, \mathbf C_{\mathrm{pr}, i}) &= p_i(\mathbf y)p_i(\mathbf x \mid \mathbf y),\\
      &= \mathcal{N}\bigl(\mathbf y; \mathbf A \mathbf m_{\mathrm{pr}, i}, \bm\Sigma_{\mathbf y} + \mathbf A \mathbf C_{\mathrm{pr}, i}\mathbf A^\top\bigr) \mathcal{N}\left(\mathbf x; \mathbf m_{\mathrm{post}, i}, \mathbf C_{\mathrm{post}, i}\right),
\end{align}
we can rewrite the numerator summand \Cref{eq:posterior-numerator} as
\begin{equation}
      p(\mathbf y \mid \mathbf x)p(\mathbf x) = \sum_{i = 1}^N w_i \,\mathcal{N}\left(\mathbf y; \mathbf A \mathbf m_{\mathrm{pr}, i}, \bm\Sigma_{\mathbf y} + \mathbf A \mathbf C_{\mathrm{pr}, i}\mathbf A^\top\right)\mathcal{N}\left(\mathbf x; \mathbf m_{\mathrm{post}, i}, \mathbf C_{\mathrm{post}, i}\right). \label{eq:final-numerator}
\end{equation}
Thus, substituting the computed numerator and denominator terms \Cref{eq:final-numerator} and \Cref{eq:final-denominator} into the original Bayes' rule expression, we have that the posterior is a Gaussian mixture with distribution
\begin{align}
      p(\mathbf x \mid \mathbf y)
      &= \sum_{i = 1}^{N} \frac{w_i\,\mathcal{N}\left(\mathbf y; \mathbf A \mathbf m_{\mathrm{pr}, i}, \bm\Sigma_{\mathbf y} + \mathbf A \mathbf C_{\mathrm{pr}, i} \mathbf A^\top\right)}{\sum_{j = 1}^{N} w_j\,\mathcal{N}\left(\mathbf y; \mathbf A \mathbf m_{\mathrm{pr}, j}, \bm\Sigma_{\mathbf y} + \mathbf A \mathbf C_{\mathrm{pr}, j}\mathbf A^\top\right)}\mathcal{N}\left(\mathbf x; \mathbf m_{\mathrm{post}, i}, \mathbf C_{\mathrm{post}, i}\right),\\
      &= \sum_{i = 1}^{N} \tilde w_i(\mathbf y)\,\mathcal{N}\left(\mathbf x; \mathbf m_{\mathrm{post}, i}, \mathbf C_{\mathrm{post}, i}\right).
\label{eq:gmm-post-unnormalized}
\end{align}
whose $i$th component has covariance, mean, and weight
\begin{align}
      \mathbf C_{\mathrm{post}, i}
      &= \left(\mathbf A^\top\bm\Sigma_{\mathbf y}^{-1}\mathbf A + \mathbf C_{\mathrm{pr}, i}^{-1}\right)^{-1}, \\
      \mathbf m_{\mathrm{post}, i}
      &= \mathbf C_{\mathrm{post}, i}
      \left(\mathbf A^\top\bm\Sigma_{\mathbf y}^{-1}\mathbf y + \mathbf C_{\mathrm{pr}, i}^{-1}\mathbf m_{\mathrm{pr}, i}\right), \\
      \tilde w_i(\mathbf y)
      &= \frac{w_i\,\mathcal{N}\left(\mathbf y; \mathbf A \mathbf m_{\mathrm{pr}, i}, \bm\Sigma_{\mathbf y} + \mathbf A \mathbf C_{\mathrm{pr}, i} \mathbf A^\top\right)}{\sum_{j = 1}^{N}w_j\,\mathcal{N}\left(\mathbf y; \mathbf A \mathbf m_{\mathrm{pr}, j}, \bm\Sigma_{\mathbf y} + \mathbf A \mathbf C_{\mathrm{pr}, j}\mathbf A^\top\right)}.
\end{align}
\end{proof}

\subsection{Gradients and scores}
\label[appendix]{sec:gradients-and-scores}

\begin{lemma}[Gaussian mixture score]
      \label[lemma]{lem:gmm-score}
      The score function of a Gaussian mixture with distribution
      \begin{equation*}
            p(\mathbf x) = \sum_{i = 1}^{N} w_i\, \mathcal{N}(\mathbf x; \mathbf m_i, \mathbf C_i)
      \end{equation*}
      is given by
      \begin{equation*}
            \nabla_{\mathbf x}\log p(\mathbf x) = -\sum_{i = 1}^{N} \tilde w_i(\mathbf x) \mathbf C_i^{-1}(\mathbf x - \mathbf m_i)
      \end{equation*}
      with updated weights
      \begin{equation*}
            \tilde w_i(\mathbf x) = \frac{w_i\, \mathcal{N}(\mathbf x; \mathbf m_i, \mathbf C_i)}{\sum_{j = 1}^{N} w_j \, \mathcal{N}(\mathbf x; \mathbf m_j, \mathbf C_j)}.
      \end{equation*}
\end{lemma}

\begin{proof}
First, by chain rule,
\begin{equation}
      \nabla_{\mathbf x}\log p(\mathbf x) = \frac{\nabla_{\mathbf x}p(\mathbf x)}{p(\mathbf x)}.
\end{equation}
The denominator is a known Gaussian mixture by proposition, thus what remains is to compute the numerator. Applying linearity of the gradient and the gradient of the Gaussian density, we obtain
\begin{align}
      \nabla_{\mathbf x}p(\mathbf x)
      &= \nabla_{\mathbf x}\sum_{i = 1}^{N} w_i\, \mathcal{N}(\mathbf x; \mathbf m_i, \mathbf C_i) \\
      &=  \sum_{i = 1}^{N} w_i\, \nabla_{\mathbf x}\mathcal{N}(\mathbf x; \mathbf m_i, \mathbf C_i) \\
      &=  - \sum_{i = 1}^{N} w_i \, \mathcal{N}(\mathbf x; \mathbf m_i, \mathbf C_i)\,\mathbf C^{-1}_i(\mathbf x - \mathbf m_i).
\end{align}
Substituting into the numerator of our chain rule expression, the score of the Gaussian mixture is given by
\begin{equation}
      \nabla_{\mathbf x}\log p(\mathbf x) = -\sum_{i = 1}^{N} \tilde w_i(\mathbf x) \mathbf C_i^{-1}(\mathbf x - \mathbf m_i)
\end{equation}
with weights
\begin{equation}
      \tilde w_i(\mathbf x) = \frac{w_i\, \mathcal{N}(\mathbf x; \mathbf m_i, \mathbf C_i)}{\sum_{j = 1}^{N} w_j \, \mathcal{N}(\mathbf x; \mathbf m_j, \mathbf C_j)}.
\end{equation}
\end{proof}

\begin{lemma}[canonical weight gradient]
      \label[lemma]{lem:weight-gradient}
      The gradient of the weight function $w_i(\mathbf x_t, t)$ defined in \Cref{lem:finite-sample-diffusion} is given by
      \begin{equation*}
            \nabla_{\mathbf x_t}w_i(\mathbf x_t, t) = \frac{\sqrt{\bar\alpha(t)}}{1-\bar\alpha(t)}w_i(\mathbf x_t, t)\bigl(\mathbf x^{(i)} - \mathbf m_{0 \mid t}(\mathbf x_t)\bigr)
      \end{equation*}
\end{lemma}

\begin{proof}
We start by recalling the definition of $w_i$ for convenience:
\begin{equation}
      w_i(\mathbf x_t, t) :=
      \frac{\mathcal{N}\Bigl(\mathbf x_t; \sqrt{\bar\alpha(t)}\mathbf x^{(i)}, \bigl(1 - \bar \alpha(t)\bigr)\mathbf I_n\Bigr)}{\sum_{j = 1}^{N}\mathcal{N}\Bigl(\mathbf x_t; \sqrt{\bar\alpha(t)}\mathbf x^{(j)}, \bigl(1 - \bar \alpha(t)\bigr)\mathbf I_n\Bigr)}
\end{equation}
For bookkeeping, define shorthand notation for the numerator and denominator of the weight function
\begin{align}
      N_i(\mathbf x_t) &:= \mathcal{N}\Bigl(\mathbf x_t; \sqrt{\bar\alpha(t)}\mathbf x^{(i)}, \bigl(1 - \bar \alpha(t)\bigr)\mathbf I_n\Bigr) \\
      D(\mathbf x_t) &:= \sum_{j = 1}^{N}\mathcal{N}\Bigl(\mathbf x_t; \sqrt{\bar\alpha(t)}\mathbf x^{(j)}, \bigl(1 - \bar \alpha(t)\bigr)\mathbf I_n\Bigr) \\
      &= \sum_{j = 1}^{N}N_j(\mathbf x_t)
\end{align}
To compute the gradient $\nabla_{\mathbf x_t}w_i(\mathbf x_t, t)$, we proceed via quotient rule
\begin{equation}
      \nabla_{\mathbf x_t}w_i(\mathbf x_t, t) = \frac{D(\mathbf x_t)\nabla_{\mathbf x_t}N_i(\mathbf x_t) - N_i(\mathbf x_t)\nabla_{\mathbf x_t}D(\mathbf x_t)}{(D(\mathbf x_t))^2}
\end{equation}
The gradient of the numerator term (for any $i$) can be computed via chain rule
\begin{align}
      \nabla_{\mathbf x_t}\log N_i(\mathbf x_t) &= \frac{\nabla_{\mathbf x_t}N_i(\mathbf x_t)}{N_i(\mathbf x_t)}, \\
      \nabla_{\mathbf x_t}N_i(\mathbf x_t)
      &= N_i(\mathbf x_t) \nabla_{\mathbf x_t}\log N_i(\mathbf x_t), \\
      &= N_i(\mathbf x_t) \nabla_{\mathbf x_t}\Bigl(-\frac{1}{2}\bigl(1-\bar\alpha(t)\bigr)^{-1}\bigl(\mathbf x_t-\sqrt{\bar\alpha(t)}\mathbf x^{(i)}\bigr)^\top \bigl(\mathbf x_t-\sqrt{\bar\alpha(t)}\mathbf x^{(i)}\bigr) + \textrm{const.}\Bigr), \\
      &= -N_i(\mathbf x_t)\bigl(1-\bar\alpha(t)\bigr)^{-1}\bigl(\mathbf x_t-\sqrt{\bar\alpha(t)}\mathbf x^{(i)}\bigr),
\end{align}
which can then be used to compute the gradient of the denominator term via linearity
\begin{align}
      \nabla_{\mathbf x_t}D(\mathbf x_t)
      &= \sum_{j = 1}^{N} \nabla_{\mathbf x_t}N_j(\mathbf x_t), \\
      &= \sum_{j = 1}^{N} -N_j(\mathbf x_t)\bigl(1-\bar\alpha(t)\bigr)^{-1}\bigl(\mathbf x_t-\sqrt{\bar\alpha(t)}\mathbf x^{(j)}\bigr), \\
      &= -\frac{1}{1-\bar\alpha(t)} \sum_{j = 1}^{N} N_j(\mathbf x_t)\bigl(\mathbf x_t-\sqrt{\bar\alpha(t)}\mathbf x^{(j)}\bigr)
\end{align}
Next, the components of numerator term in the \emph{gradient} of the weight function $w_i$ are computable as
\begin{align}
      D(\mathbf x_t)\nabla_{\mathbf x_t}N_i(\mathbf x_t)
      &= \left(\sum_{j = 1}^{N}N_j(\mathbf x_t)\right)\cdot \left(-N_i(\mathbf x_t)\bigl(1-\bar\alpha(t)\bigr)^{-1}\bigl(\mathbf x_t-\sqrt{\bar\alpha(t)}\mathbf x^{(i)}\bigr)\right) \\
      &= - \frac{N_i(\mathbf x_t)\bigl(\mathbf x_t-\sqrt{\bar\alpha(t)}\mathbf x^{(i)}\bigr)}{1-\bar\alpha(t)}\sum_{j = 1}^{N}N_j(\mathbf x_t), \\
      N_i(\mathbf x_t) \nabla_{\mathbf x_t}D(\mathbf x_t)
      &= -\frac{N_i(\mathbf x_t)}{1-\bar\alpha(t)}\sum_{j = 1}^{N} N_j(\mathbf x_t)\bigl(\mathbf x_t-\sqrt{\bar\alpha(t)}\mathbf x^{(j)}\bigr).
\end{align}
Taking the difference, the entire numerator simplifies as
\begin{align}
      &D(\mathbf x_t)\nabla_{\mathbf x_t}N_i(\mathbf x_t) - N_i(\mathbf x_t) \nabla_{\mathbf x_t}D(\mathbf x_t) \\
      &=- \frac{N_i(\mathbf x_t)\bigl(\mathbf x_t-\sqrt{\bar\alpha(t)}\mathbf x^{(i)}\bigr)}{1-\bar\alpha(t)}\sum_{j = 1}^{N}N_j(\mathbf x_t) + \frac{N_i(\mathbf x_t)}{1-\bar\alpha(t)}\sum_{j = 1}^{N} N_j(\mathbf x_t)\bigl(\mathbf x_t-\sqrt{\bar\alpha(t)}\mathbf x^{(j)}\bigr) \\
      &=- \frac{N_i(\mathbf x_t)}{1-\bar\alpha(t)}\sum_{j = 1}^{N}N_j(\mathbf x_t)\bigl(\mathbf x_t-\sqrt{\bar\alpha(t)}\mathbf x^{(i)}\bigr) + \frac{N_i(\mathbf x_t)}{1-\bar\alpha(t)}\sum_{j = 1}^{N} N_j(\mathbf x_t)\bigl(\mathbf x_t-\sqrt{\bar\alpha(t)}\mathbf x^{(j)}\bigr) \\
      &= \frac{N_i(\mathbf x_t)}{1-\bar\alpha(t)}\sum_{j = 1}^{N} N_j(\mathbf x_t)\left[\bigl(\mathbf x_t-\sqrt{\bar\alpha(t)}\mathbf x^{(j)}\bigr) - \bigl(\mathbf x_t-\sqrt{\bar\alpha(t)}\mathbf x^{(i)}\bigr)\right] \\
      &= \frac{\sqrt{\bar\alpha(t)}N_i(\mathbf x_t)}{1-\bar\alpha(t)}\sum_{j = 1}^{N} N_j(\mathbf x_t) (\mathbf x^{(i)} - \mathbf x^{(j)}).
\end{align}
Substituting the numerator into the quotient rule expression, and simply by recalling that $w_i(\mathbf x_t, t) = N_i(\mathbf x_t) / D(\mathbf x_t)$,
\begin{align}
      \nabla_{\mathbf x_t}w_i(\mathbf x_t, t)
      &= \frac{D(\mathbf x_t)\nabla_{\mathbf x_t}N_i(\mathbf x_t) - N_i(\mathbf x_t) \nabla_{\mathbf x_t}D(\mathbf x_t)}{(D(\mathbf x_t))^2} \\
      &= \frac{1}{(D(\mathbf x_t))^2} \left(\frac{\sqrt{\bar\alpha(t)}N_i(\mathbf x_t)}{1-\bar\alpha(t)}\sum_{j = 1}^{N} N_j(\mathbf x_t)(\mathbf x^{(i)} - \mathbf x^{(j)})\right)\\
      &= \frac{\sqrt{\bar\alpha(t)}N_i(\mathbf x_t)/D(\mathbf x_t)}{1-\bar\alpha(t)}\sum_{j = 1}^{N} \frac{N_j(\mathbf x_t)}{D(\mathbf x_t)}(\mathbf x^{(i)} - \mathbf x^{(j)}) \\
      &= \frac{\sqrt{\bar\alpha(t)}}{1-\bar\alpha(t)}w_i(\mathbf x_t, t)\sum_{j = 1}^{N} w_j(\mathbf x_t, t) (\mathbf x^{(i)} - \mathbf x^{(j)}).
\end{align}
Rearranging, recalling that $\sum w_i = 1$, and applying the definition of $\mathbf m_{0 \mid t}$ gives the desired result.
\begin{align}
      \nabla_{\mathbf x_t}w_i(\mathbf x_t, t)
      &= \frac{\sqrt{\bar\alpha(t)}}{1-\bar\alpha(t)}w_i(\mathbf x_t, t)\sum_{j = 1}^{N} w_j(\mathbf x_t, t) \mathbf x^{(i)} - w_j(\mathbf x_t, t)\mathbf x^{(j)} \\
      &= \frac{\sqrt{\bar\alpha(t)}}{1-\bar\alpha(t)}w_i(\mathbf x_t, t)\Biggl(\mathbf x^{(i)}\underbrace{\sum_{j = 1}^{N} w_j(\mathbf x_t, t)}_1  - \underbrace{\sum_{j = 1}^{N}w_j(\mathbf x_t, t)\mathbf x^{(j)}}_{\mathbf m_{0 \mid t}}\Biggr)\\
      &= \frac{\sqrt{\bar\alpha(t)}}{1-\bar\alpha(t)}w_i(\mathbf x_t, t)\bigl(\mathbf x^{(i)} - \mathbf m_{0\mid t}\bigr)
\end{align}

\end{proof}

\setcounter{theorem}{0}

\section{Finite sample regime proofs}
\label[appendix]{sec:fsr-proofs}
\subsection{Proof of \Cref{lem:finite-sample-diffusion}}
\label[appendix]{proof:finite-sample-marginal}

\begin{proof}
By the diffusion forward transition, the probability of  state $\mathbf x_t$ at time $t$ given initial condition $\mathbf x_0$ is given by
\begin{equation}
    \label{eq:prop1-vpsde-likelihood}
    p_t(\mathbf x_t \mid \mathbf x_0) = \mathcal{N}\Bigl(\mathbf x_t; \sqrt{\bar\alpha(t)}\mathbf x_0, \bigl(1 - \bar \alpha(t)\bigr)\mathbf I_n\Bigr).
\end{equation}
Applying \Cref{lem:gaussian-discrete-marginal} with likelihood \Cref{eq:prop1-vpsde-likelihood} and prior $p_{\mathrm{pr}}^N$ gives the desired marginal.
\begin{equation}
    p_t(\mathbf x_t) = \frac{1}{N}\sum_{i = 1}^{N}\mathcal{N}\Bigl(\mathbf x_t; \sqrt{\bar\alpha(t)}\mathbf x^{(i)}, \bigl(1 - \bar \alpha(t)\bigr)\mathbf I_n\Bigr).
\end{equation}
Similarly, applying Lemma~\ref{lem:gaussian-discrete-posterior} with likelihood \Cref{eq:prop1-vpsde-likelihood} and prior $p_{\mathrm{pr}}^N$ gives the desired posterior.
\end{proof}

\subsection{Proof of \Cref{thm:finite-sample-likelihood}}
\label[appendix]{proof:finite-sample-likelihood}

\begin{proof}
We start by writing the likelihood $p_{\mathbf y \mid t}(\mathbf y \mid \mathbf x_t)$ in marginal form
\begin{equation}
    p_{\mathbf y \mid t}(\mathbf y \mid \mathbf x_t) = \int p_{\mathbf y \mid 0, t}(\mathbf y\mid \mathbf x_0, \mathbf x_t)p_{0 \mid t}(\mathbf x_0 \mid \mathbf x_t)\mathrm{d}\mathbf x_0.
\end{equation}
The likelihood kernel $p(\mathbf y\mid \mathbf x_0, \mathbf x_t)$ can be rewritten via Markov property and the definition of the forward model.
\begin{equation}
    p(\mathbf y\mid \mathbf x_0, \mathbf x_t) = p_{\mathbf y \mid 0}(\mathbf y \mid \mathbf x_0) = \mathcal{N}(\mathbf y; \mathcal{A}(\mathbf x_0), \bm \Sigma_{\mathbf y}).
\end{equation}
By \Cref{lem:finite-sample-diffusion}, the diffusion process posterior $p_{0 \mid t}(\mathbf x_0 \mid \mathbf x_t)$ is a discrete measure
\begin{equation}
    p_{0 \mid t}(\mathbf x_0 \mid \mathbf x_t) = \sum_{i = 1}^{N} w_i(\mathbf x_t, t) \delta(\mathbf x_0 - \mathbf x^{(i)})
\end{equation}
with weights $w_i(\mathbf x_t, t)$ as defined in \Cref{lem:finite-sample-diffusion}.

Substituting these two facts into the integral form of the likelihood, we obtain a Gaussian mixture whose means are constant w.r.t.\ $\mathbf y$ and $\mathbf x_t$ (i.e., of identical form to the case when $\mathcal{A}$ is linear).
\begin{align}
     p_{\mathbf y \mid t}(\mathbf y \mid \mathbf x_t)
     &= \int \mathcal{N}(\mathbf y; \mathcal{A}(\mathbf x_0), \bm \Sigma_{\mathbf y}) \sum_{i = 1}^{N} w_i(\mathbf x_t, t) \delta(\mathbf x_0 - \mathbf x^{(i)}) \mathrm{d}\mathbf x_0\\
     &= \sum_{i = 1}^{N} w_i(\mathbf x_t, t) \int \mathcal{N}(\mathbf y; \mathcal{A}(\mathbf x_0), \bm \Sigma_{\mathbf y}) \delta(\mathbf x_0 - \mathbf x^{(i)}) \mathrm{d}\mathbf x_0 \\
     &= \sum_{i = 1}^{N} w_i(\mathbf x_t, t) \,\mathcal{N}\Bigl(\mathbf y; \mathcal{A}\bigl(\mathbf x^{(i)}\bigr), \bm \Sigma_{\mathbf y}\Bigr)
\end{align}
\end{proof}

\subsection{Proof of \Cref{thm:finite-sample-posterior}}
\label[appendix]{proof:finite-sample-posterior}

\begin{proof}
    We start by writing the posterior $p^N_{t \mid \mathbf y}(\mathbf x_t \mid \mathbf y)$ via Bayes's rule
    \begin{equation}
        p^N_{t \mid \mathbf y}(\mathbf x_t \mid \mathbf y) = \frac{p^N_{\mathbf y \mid t}(\mathbf y \mid \mathbf x_t)p^N_t(\mathbf x_t)}{\int_{\mathcal{X}} p_{\mathbf y \mid 0}(\mathbf y \mid \mathbf x_0) p_{\mathrm{pr}}^N(\mathbf x_0) \mathrm{d}\mathbf x_0}
    \end{equation}
    The terms in the numerator are known from \Cref{thm:finite-sample-likelihood} and \Cref{lem:finite-sample-diffusion} respectively, and the denominator terms are given by the measurement model and the finite-sample regime respectively. Performing direct substitution and rearranging,
\begin{align}
p^N_{t \mid \mathbf y}(\mathbf x_t \mid \mathbf y)
&=
\frac{
\left(\sum_{i=1}^{N} w_i(\mathbf x_t,t)\,
\mathcal{N}\Bigl(\mathbf y;\mathcal{A}(\mathbf x^{(i)}),\bm\Sigma_{\mathbf y}\Bigr)\right)
\left(\frac{1}{N}\sum_{i=1}^{N}
\mathcal{N}\Bigl(\mathbf x_t;\sqrt{\bar\alpha(t)}\mathbf x^{(i)},
(1-\bar\alpha(t))\mathbf I_n\Bigr)\right)
}{
\frac1N\sum_{j=1}^{N}
\mathcal{N}\Bigl(\mathbf y;\mathcal{A}(\mathbf x^{(j)}),\bm\Sigma_{\mathbf y}\Bigr)
}
\\
&=
\frac{
\frac1N
\sum_{i=1}^{N}
\mathcal{N}\Bigl(\mathbf y;\mathcal{A}(\mathbf x^{(i)}),\bm\Sigma_{\mathbf y}\Bigr)
\mathcal{N}\Bigl(\mathbf x_t;\sqrt{\bar\alpha(t)}\mathbf x^{(i)},
(1-\bar\alpha(t))\mathbf I_n\Bigr)
}{
\frac1N\sum_{j=1}^{N}
\mathcal{N}\Bigl(\mathbf y;\mathcal{A}(\mathbf x^{(j)}),\bm\Sigma_{\mathbf y}\Bigr)
}
\\
&=
\sum_{i=1}^{N}
\frac{
\mathcal{N}\Bigl(\mathbf y;\mathcal{A}(\mathbf x^{(i)}),\bm\Sigma_{\mathbf y}\Bigr)
}{
\sum_{j=1}^{N}
\mathcal{N}\Bigl(\mathbf y;\mathcal{A}(\mathbf x^{(j)}),\bm\Sigma_{\mathbf y}\Bigr)
}
\mathcal{N}\Bigl(\mathbf x_t;\sqrt{\bar\alpha(t)}\mathbf x^{(i)},
(1-\bar\alpha(t))\mathbf I_n\Bigr).
\end{align}
\end{proof}

\section{Experimental details}
\label[appendix]{sec:experiemental-details}
\subsection{Target objects}
\label[appendix]{sec:target-objects}

In this section, we derive closed-form unconditional and conditional diffusion distributions and score functions for a variety of priors and measurement operators. For every prior we report five \emph{unconditional objects}: the marginal $p_t$, marginal score $\nabla_{\mathbf x_t} \log p_t(\mathbf x_t)$, denoiser $p_{0 \mid t}$, denoiser mean $\mathbf m_{0 \mid t}(\mathbf x_t)$, and denoiser covariance $\mathbf C_{0 \mid t}(\mathbf x_t)$. Unconditional objects are known analytically in all three cases independent of the forward model. Next, we additionally report four \emph{conditional objects}: the likelihood $p_{\mathbf y \mid t}$, likelihood score $\nabla_{\mathbf x_t} \log p_{\mathbf y \mid t}$, posterior $p_{t \mid \mathbf y}$, and posterior score $\nabla_{\mathbf x_t} \log p_{t \mid \mathbf y}$ only in contexts where they are accessible. For all settings, we assume VP-SDE forward transition
\begin{equation}
    p_{t \mid 0}(\mathbf x_t \mid \mathbf x_0) = \mathcal{N}(\mathbf x_t; \sqrt{\bar\alpha(t)} \mathbf x_0, (1 - \bar\alpha(t))\mathbf I_n),
\end{equation}
and additive Gaussian measurement noise
\begin{equation}
  p_{\mathbf y \mid 0}(\mathbf y \mid \mathbf x_0) = \mathcal{N}(\mathbf y; \mathcal{A}(\mathbf x_0), \bm \Sigma_{\mathbf y}).
\end{equation}

\subsubsection{Discrete prior}
\label[appendix]{sec:discrete-prior-analysis}

Let $p_{\mathrm{pr}}(\mathbf x_0)$ denote a discrete measure assigning positive probability $p_i > 0$ to finitely-many atoms $\mathbf x^{(i)}$:
\begin{equation}
    p_{\mathrm{pr}}(\mathbf x_0) = \sum_{i = 1}^N p_i \, \delta(\mathbf x_0 - \mathbf x^{(i)})
\end{equation}

\paragraph{Unconditional objects.}
The marginal $p_t(\mathbf x_t)$ at diffusion time $t$ can be derived by direct substitution, exchanging the integral and the finite sum, and integrating the mean of the Gaussian density against the Dirac measure $\delta$.
\begin{align}
    p_t(\mathbf x_t)
    &= \int_{\mathcal X} p_{t \mid 0}(\mathbf x_t \mid \mathbf x_0)\, p_{\mathrm{pr}} (\mathbf x_0)\, \mathrm{d}\mathbf x_0 \\
    &=\int_{\mathcal X} \mathcal{N}\bigl(\mathbf x_t; \sqrt{\bar\alpha(t)} \mathbf x_0, (1-\bar\alpha(t))\mathbf I_n\bigr)\,  \sum_{i = 1}^N p_i \delta(\mathbf x_0 - \mathbf x^{(i)})\, \mathrm{d}\mathbf x_0\\
    &= \sum_{i = 1}^N p_i \int_{\mathcal X} \mathcal{N}\bigl(\mathbf x_t; \sqrt{\bar\alpha(t)} \mathbf x_0, (1-\bar\alpha(t))\mathbf I_n\bigr)\, \delta(\mathbf x_0 - \mathbf x^{(i)})\, \mathrm{d}\mathbf x_0\\
    &= \sum_{i = 1}^N p_i \,\mathcal{N}\bigl(\mathbf x_t; \sqrt{\bar\alpha (t)} \mathbf x^{(i)}, (1 - \bar\alpha(t))\mathbf I_n\bigr).
\end{align}
Because the marginal is a Gaussian mixture, the Gaussian mixture score identity (\Cref{lem:gmm-score}) yields
\begin{equation}
    \nabla_{\mathbf x_t} \log p_t(\mathbf x_t) =
    -\sum_{i = 1}^{N} \tilde p_i(\mathbf x_t, t) \frac{\mathbf x_t - \sqrt{\bar\alpha (t)}\mathbf x^{(i)}}{1 - \bar\alpha(t)}
\end{equation}
\begin{equation}
    \tilde p_{i} (\mathbf x_t, t) = \frac{p_i \,\mathcal{N}\bigl(\mathbf x_t; \sqrt{\bar\alpha (t)} \mathbf x^{(i)}, (1 - \bar\alpha(t))\mathbf I_n\bigr)}{\sum_{j = 1}^N p_j \,\mathcal{N}\bigl(\mathbf x_t; \sqrt{\bar\alpha (t)} \mathbf x^{(j)}, (1 - \bar\alpha(t))\mathbf I_n\bigr)}.
\end{equation}

Once the marginal $p_t$ is known, the denoiser $p_{0 \mid t}$ follows from Bayes's rule, direct substitution, and rearrangement:
\begin{align}
    p_{0 \mid t}(\mathbf x_0 \mid \mathbf x_t)
    &= \frac{p_{t \mid 0}(\mathbf x_t \mid \mathbf x_0)\, p_{\mathrm{pr}}(\mathbf x_0)}{p_t(\mathbf x_t)} \\
    &= \frac{\mathcal{N}(\mathbf x_t; \sqrt{\bar\alpha(t)} \mathbf x_0, (1 - \bar\alpha(t))\mathbf I_n)\, \sum_{i = 1}^N p_i \, \delta(\mathbf x_0 - \mathbf x^{(i)})}{\sum_{j = 1}^N p_j \,\mathcal{N}\bigl(\mathbf x_t; \sqrt{\bar\alpha (t)} \mathbf x^{(j)}, (1 - \bar\alpha(t))\mathbf I_n\bigr)} \\
    &= \sum_{i = 1}^N\frac{p_i\mathcal{N}(\mathbf x_t; \sqrt{\bar\alpha(t)} \mathbf x_0, (1 - \bar\alpha(t))\mathbf I_n)}{\sum_{j = 1}^N p_j \,\mathcal{N}\bigl(\mathbf x_t; \sqrt{\bar\alpha (t)} \mathbf x^{(j)}, (1 - \bar\alpha(t))\mathbf I_n\bigr)} \delta(\mathbf x_0 - \mathbf x^{(i)})\\
    &= \sum_{i = 1}^N \tilde p_i(\mathbf x_t, t)\, \delta(\mathbf x_0 - \mathbf x^{(i)}).
\end{align}
Because the denoiser is a discrete measure, its mean is simply a weighted average of the atoms:
\begin{align}
    \mathbf m_{0 \mid t}(\mathbf x_t) &= \sum_{i = 1}^N \tilde p_i(\mathbf x_t, t)\, \mathbf x^{(i)}.
\end{align}
Furthermore, its covariance is a weighted spread over pairs of atoms,
\begin{align}
  \mathbf C_{0 \mid t}(\mathbf x_t) 
  &= \sum_{i = 1}^N \tilde p_i(\mathbf x_t, t) (\mathbf x^{(i)} - \mathbf m_{0 \mid t}(\mathbf x_t))(\mathbf x^{(i)} - \mathbf m_{0 \mid t}(\mathbf x_t))^\top \\
  &= \sum_{i = 1}^N \tilde p_i(\mathbf x_t, t)\,
     \left[
       \mathbf x^{(i)}{\mathbf x^{(i)}}^\top
       - \mathbf x^{(i)}\mathbf m_{0 \mid t}(\mathbf x_t)^\top
       - \mathbf m_{0 \mid t}(\mathbf x_t){\mathbf x^{(i)}}^\top
       + \mathbf m_{0 \mid t}(\mathbf x_t)\mathbf m_{0 \mid t}(\mathbf x_t)^\top
     \right] \\
  &= \sum_{i = 1}^N \tilde p_i(\mathbf x_t, t)\,
       \mathbf x^{(i)}{\mathbf x^{(i)}}^\top
     - \left(\sum_{i = 1}^N \tilde p_i(\mathbf x_t, t)\mathbf x^{(i)}\right)
       \mathbf m_{0 \mid t}(\mathbf x_t)^\top \\
  &\qquad
     - \mathbf m_{0 \mid t}(\mathbf x_t)
       \left(\sum_{i = 1}^N \tilde p_i(\mathbf x_t, t)\mathbf x^{(i)}\right)^\top
     + \left(\sum_{i = 1}^N \tilde p_i(\mathbf x_t, t)\right)
       \mathbf m_{0 \mid t}(\mathbf x_t)\mathbf m_{0 \mid t}(\mathbf x_t)^\top \\
  &= \sum_{i = 1}^N \tilde p_i(\mathbf x_t, t)\,
       \mathbf x^{(i)}{\mathbf x^{(i)}}^\top
     - \mathbf m_{0 \mid t}(\mathbf x_t)\mathbf m_{0 \mid t}(\mathbf x_t)^\top
     - \mathbf m_{0 \mid t}(\mathbf x_t)\mathbf m_{0 \mid t}(\mathbf x_t)^\top
     + \mathbf m_{0 \mid t}(\mathbf x_t)\mathbf m_{0 \mid t}(\mathbf x_t)^\top \\
  &= \sum_{i = 1}^N \tilde p_i(\mathbf x_t, t)\,
       \mathbf x^{(i)}{\mathbf x^{(i)}}^\top
     - \mathbf m_{0 \mid t}(\mathbf x_t)\mathbf m_{0 \mid t}(\mathbf x_t)^\top \\
  &= \sum_{i = 1}^N \tilde p_i(\mathbf x_t, t)\, \mathbf x^{(i)}{\mathbf x^{(i)}}^\top
     - \sum_{i,j = 1}^N \tilde p_i(\mathbf x_t, t)\tilde p_j(\mathbf x_t, t)\,
       \mathbf x^{(i)}{\mathbf x^{(j)}}^\top \\
  &= \frac{1}{2}\sum_{i,j = 1}^N \tilde p_i(\mathbf x_t, t)\tilde p_j(\mathbf x_t, t)\,
     \left(
       \mathbf x^{(i)}{\mathbf x^{(i)}}^\top
       + \mathbf x^{(j)}{\mathbf x^{(j)}}^\top
       - \mathbf x^{(i)}{\mathbf x^{(j)}}^\top
       - \mathbf x^{(j)}{\mathbf x^{(i)}}^\top
     \right) \\
  &= \frac{1}{2}\sum_{i, j = 1}^N \tilde p_i(\mathbf x_t, t)\, \tilde p_j(\mathbf x_t, t)\, \bigl(\mathbf x^{(i)} - \mathbf x^{(j)}\bigr)\bigl(\mathbf x^{(i)} - \mathbf x^{(j)}\bigr)^\top
\end{align}

\paragraph{Conditional objects.}
Because the denoiser is a discrete measure, the marginal integral collapses to a finite sum.
\begin{align}
    p_{\mathbf y \mid t}(\mathbf y \mid \mathbf x_t)
    &= \int_{\mathcal X} p_{\mathbf y \mid 0}(\mathbf y \mid \mathbf x_0)\, p_{0 \mid t}(\mathbf x_0 \mid \mathbf x_t) \,\mathrm{d}\mathbf x_0 \\
    &= \int_{\mathcal X} \mathcal{N}(\mathbf y; \mathcal{A}(\mathbf x_0), \bm \Sigma_{\mathbf y})\,\sum_{i = 1}^N \tilde p_i(\mathbf x_t, t)\, \delta(\mathbf x_0 - \mathbf x^{(i)}) \,\mathrm{d}\mathbf x_0 \\
    &= \sum_{i = 1}^N \tilde p_i(\mathbf x_t, t)\int_{\mathcal X} \mathcal{N}(\mathbf y; \mathcal{A}(\mathbf x_0), \bm \Sigma_{\mathbf y})\, \delta(\mathbf x_0 - \mathbf x^{(i)}) \,\mathrm{d}\mathbf x_0 \\
    &= \sum_{i = 1}^N \tilde p_i(\mathbf x_t, t)\, \mathcal{N}\bigl(\mathbf y; \mathcal{A}(\mathbf x^{(i)}), \bm\Sigma_{\mathbf y}\bigr).
\end{align}
Differentiating, and observing only the weights $\tilde p_i(\mathbf x_t, t)$ depend on $\mathbf x_t$,
\begin{align}
    \nabla_{\mathbf x_t} \log p_{\mathbf y \mid t}(\mathbf y \mid \mathbf x_t)
    &= -\sum_{i = 1}^N \frac{\tilde p_i(\mathbf x_t, t)\, \mathcal{N}\bigl(\mathbf y; \mathcal{A}(\mathbf x^{(i)}), \bm\Sigma_{\mathbf y}\bigr)}{\sum_{j = 1}^N \tilde p_j(\mathbf x_t, t)\, \mathcal{N}\bigl(\mathbf y; \mathcal{A}(\mathbf x^{(j)}), \bm\Sigma_{\mathbf y}\bigr)}\frac{\mathbf x_t - \sqrt{\bar\alpha(t)}\mathbf x^{(i)}}{1 - \bar\alpha(t)} \nonumber \\
    &\phantom{{}={}} + \underbrace{\sum_{i = 1}^N \tilde p_i(\mathbf x_t, t) \frac{\mathbf x_t - \sqrt{\bar\alpha(t)}\mathbf x^{(i)}}{1 - \bar\alpha(t)}}_{- \nabla_{\mathbf x_t} \log p_t(\mathbf x_t)}.
\end{align}
By Bayes's rule and rearrangement, the posterior is the Gaussian mixture:
\begin{align}
    p_{t \mid \mathbf y}(\mathbf x_t \mid \mathbf y) 
    &= \frac{p_{\mathbf y \mid t}(\mathbf y \mid \mathbf x_t) p_t(\mathbf x_t)}{\int_{\mathcal X} p_{\mathbf y \mid 0}(\mathbf y \mid \mathbf x_0)p_{\mathrm pr}(\mathbf x_0)\,\mathrm{d}\mathbf x_0}  \\
    &= \sum_{i = 1}^N \frac{p_i\, \mathcal{N}\bigl(\mathbf y; \mathcal{A}(\mathbf x^{(i)}), \bm\Sigma_{\mathbf y}\bigr)}{\sum_{j = 1}^N p_j\, \mathcal{N}\bigl(\mathbf y; \mathcal{A}(\mathbf x^{(j)}), \bm\Sigma_{\mathbf y}\bigr)}\, \mathcal{N}\bigl(\mathbf x_t; \sqrt{\bar\alpha(t)}\mathbf x^{(i)}, (1 - \bar\alpha(t))\mathbf I_n\bigr),
\end{align}
and the posterior score collapses to
\begin{equation}
    \nabla_{\mathbf x_t} \log p_{t \mid \mathbf y}(\mathbf x_t \mid \mathbf y) = -\sum_{i = 1}^N \frac{\tilde p_i(\mathbf x_t, t)\, \mathcal{N}\bigl(\mathbf y; \mathcal{A}(\mathbf x^{(i)}), \bm\Sigma_{\mathbf y}\bigr)}{\sum_{j = 1}^N \tilde p_j(\mathbf x_t, t)\, \mathcal{N}\bigl(\mathbf y; \mathcal{A}(\mathbf x^{(j)}), \bm\Sigma_{\mathbf y}\bigr)}\frac{\mathbf x_t - \sqrt{\bar\alpha(t)}\mathbf x^{(i)}}{1 - \bar\alpha(t)}.
\end{equation}

\subsubsection{Gaussian prior}
\label[appendix]{sec:gaussian-prior-analysis}

Let $p_{\mathrm{pr}}(\mathbf x_0)$ denote a (unimodal) Gaussian with mean vector $ \mathbf m_{\mathrm{pr}}$ and (positive definite) covariance matrix $\mathbf C_{\mathrm{pr}}$:
\begin{equation}
    p_{\mathrm{pr}}(\mathbf x_0) = \mathcal{N}\bigl(\mathbf x_0; \mathbf m_{\mathrm{pr}}, \mathbf C_{\mathrm{pr}}\bigr)
\end{equation}

\paragraph{Unconditional objects.}
The marginal $p_t$ is Gaussian by \Cref{lem:gaussian-gaussian-marginal}:
\begin{align}
    p_t(\mathbf x_t)
    &= \int_{\mathcal X} \mathcal{N}\bigl(\mathbf x_t; \sqrt{\bar\alpha(t)} \mathbf x_0, (1-\bar\alpha(t))\mathbf I_n\bigr)\, \mathcal{N}\bigl(\mathbf x_0; \mathbf m_{\mathrm{pr}}, \mathbf C_{\mathrm{pr}}\bigr)\, \mathrm{d}\mathbf x_0 \\
    &= \mathcal{N}\bigl(\mathbf x_t; \sqrt{\bar\alpha(t)} \mathbf m_{\mathrm{pr}},\, \bar\alpha(t) \mathbf C_{\mathrm{pr}} + (1-\bar\alpha(t))\mathbf I_n\bigr),
\end{align}
hence
\begin{equation}
    \nabla_{\mathbf x_t} \log p_t(\mathbf x_t)
    =  -\bigl(\bar\alpha(t) \mathbf C_{\mathrm{pr}} + (1-\bar\alpha(t))\mathbf I_n\bigr)^{-1} \bigl(\mathbf x_t - \sqrt{\bar\alpha(t)}\, \mathbf m_{\mathrm{pr}}\bigr).
\end{equation}
By \Cref{lem:gaussian-gaussian-conjugacy}, the denoiser is Gaussian,
\begin{equation}
    p_{0 \mid t}(\mathbf x_0 \mid \mathbf x_t) = \mathcal{N}\bigl(\mathbf x_0; \mathbf m_{0 \mid t}(\mathbf x_t), \mathbf C_{0 \mid t}\bigr),
\end{equation}
\begin{align}
    \mathbf C_{0 \mid t} &= \Bigl(\mathbf C_{\mathrm{pr}}^{-1} + \tfrac{\bar\alpha(t)}{1-\bar\alpha(t)}\mathbf I_n\Bigr)^{-1}, \\
    \mathbf m_{0 \mid t}(\mathbf x_t) &= \mathbf C_{0 \mid t} \Bigl(\mathbf C_{\mathrm{pr}}^{-1}\mathbf m_{\mathrm{pr}} + \tfrac{\sqrt{\bar\alpha(t)}}{1 - \bar\alpha(t)}\mathbf x_t\Bigr).
\end{align}
The denoiser covariance is independent of $\mathbf x_t$, reflecting the linear-Gaussian structure of the prior.

\paragraph{Conditional objects.}
The likelihood is
\begin{equation}
    p_{\mathbf y \mid t}(\mathbf y \mid \mathbf x_t) = \int_{\mathcal X} \mathcal{N}\bigl(\mathbf y; \mathcal{A}(\mathbf x_0), \bm\Sigma_{\mathbf y}\bigr)\, \mathcal{N}\bigl(\mathbf x_0; \mathbf m_{0 \mid t}, \mathbf C_{0 \mid t}\bigr)\,\mathrm{d}\mathbf x_0,
\end{equation}
which has no closed form for general $\mathcal{A}$. For affine $\mathcal{A}(\mathbf x_0) = \mathbf A \mathbf x_0 + \mathbf b$, \Cref{lem:gaussian-gaussian-marginal} gives
\begin{equation}
    p_{\mathbf y \mid t}(\mathbf y \mid \mathbf x_t) = \mathcal{N}\bigl(\mathbf y;\, \mathbf A\mathbf m_{0 \mid t} + \mathbf b,\, \bm\Sigma_{\mathbf y} + \mathbf A\mathbf C_{0 \mid t}\mathbf A^\top \bigr).
\end{equation}
The likelihood score follows from chain rule on $\mathbf m_{0 \mid t}(\mathbf x_t)$:
\begin{equation}
    \nabla_{\mathbf x_t} \log p_{\mathbf y \mid t}(\mathbf y \mid \mathbf x_t)
    = \tfrac{\sqrt{\bar\alpha(t)}}{1 - \bar\alpha(t)}\, \mathbf C_{0 \mid t}\, \mathbf A^\top \bigl(\bm\Sigma_{\mathbf y} + \mathbf A\mathbf C_{0 \mid t}\mathbf A^\top\bigr)^{-1} \bigl(\mathbf y - \mathbf A\mathbf m_{0 \mid t} - \mathbf b\bigr).
\end{equation}

For the posterior we work with the joint $(\mathbf x_t, \mathbf y)$. Since $\mathbf x_t = \sqrt{\bar\alpha(t)}\mathbf x_0 + \bm\eta_1$ and $\mathbf y = \mathbf A \mathbf x_0 + \mathbf b + \bm\eta_2$ with independent Gaussian noise $\bm\eta_1 \sim \mathcal{N}(\mathbf 0, (1-\bar\alpha(t))\mathbf I_n)$ and $\bm\eta_2 \sim \mathcal{N}(\mathbf 0, \bm\Sigma_{\mathbf y})$, the joint is Gaussian with
\begin{equation}
    p(\mathbf x_t, \mathbf y) = \mathcal{N}\!\left(\begin{pmatrix}\mathbf x_t\\\mathbf y\end{pmatrix}; \begin{pmatrix}\sqrt{\bar\alpha(t)}\,\mathbf m_{\mathrm{pr}} \\ \mathbf A\mathbf m_{\mathrm{pr}} + \mathbf b\end{pmatrix},\, \bm\Lambda\right),
\end{equation}
\begin{equation}
    \bm\Lambda = \begin{pmatrix}\bar\alpha(t)\mathbf C_{\mathrm{pr}} + (1-\bar\alpha(t))\mathbf I_n & \sqrt{\bar\alpha(t)}\,\mathbf C_{\mathrm{pr}}\mathbf A^\top \\ \sqrt{\bar\alpha(t)}\,\mathbf A\mathbf C_{\mathrm{pr}} & \mathbf A\mathbf C_{\mathrm{pr}}\mathbf A^\top + \bm\Sigma_{\mathbf y}\end{pmatrix}.
\end{equation}
Conditioning on $\mathbf y$ yields a Gaussian posterior $p_{t \mid \mathbf y}(\mathbf x_t \mid \mathbf y) = \mathcal{N}\bigl(\mathbf x_t; \mathbf m_{\mathrm{post}}(\mathbf y), \mathbf C_{\mathrm{post}}\bigr)$ with
\begin{align}
    \mathbf m_{\mathrm{post}}(\mathbf y) &= \sqrt{\bar\alpha(t)}\,\mathbf m_{\mathrm{pr}} + \sqrt{\bar\alpha(t)}\,\mathbf C_{\mathrm{pr}}\mathbf A^\top\bigl(\mathbf A\mathbf C_{\mathrm{pr}}\mathbf A^\top + \bm\Sigma_{\mathbf y}\bigr)^{-1}(\mathbf y - \mathbf A\mathbf m_{\mathrm{pr}} - \mathbf b), \\
    \mathbf C_{\mathrm{post}} &= \bar\alpha(t)\mathbf C_{\mathrm{pr}} + (1-\bar\alpha(t))\mathbf I_n - \bar\alpha(t)\mathbf C_{\mathrm{pr}}\mathbf A^\top\bigl(\mathbf A\mathbf C_{\mathrm{pr}}\mathbf A^\top + \bm\Sigma_{\mathbf y}\bigr)^{-1}\mathbf A\mathbf C_{\mathrm{pr}}.
\end{align}
The posterior score is linear in $\mathbf x_t$,
\begin{equation}
    \nabla_{\mathbf x_t} \log p_{t \mid \mathbf y}(\mathbf x_t \mid \mathbf y) = -\mathbf C_{\mathrm{post}}^{-1}\bigl(\mathbf x_t - \mathbf m_{\mathrm{post}}(\mathbf y)\bigr).
\end{equation}

\subsubsection{Gaussian mixture prior}
\label[appendix]{sec:gmm-prior-analysis}

Let $p_{\mathrm{pr}}(\mathbf x_0)$ denote a Gaussian mixture where the $i$th component has weight $p_i$, mean vector $ \mathbf m_{\mathrm{pr}, i}$ and (positive definite) covariance matrix $\mathbf C_{\mathrm{pr}, i}$:
\begin{equation}
    p_{\mathrm{pr}}(\mathbf x_0) = \sum_{i = 1}^N p_i\, \mathcal{N}\bigl(\mathbf x_0; \mathbf m_{\mathrm{pr}, i}, \mathbf C_{\mathrm{pr}, i}\bigr)
\end{equation}

\paragraph{Unconditional objects.}
The marginal is the Gaussian mixture
\begin{align}
    p_t(\mathbf x_t)
    &= \sum_{i = 1}^N p_i \int_{\mathcal X} \mathcal{N}\bigl(\mathbf x_t; \sqrt{\bar\alpha(t)} \mathbf x_0, (1-\bar\alpha(t))\mathbf I_n\bigr)\, \mathcal{N}\bigl(\mathbf x_0; \mathbf m_{\mathrm{pr}, i}, \mathbf C_{\mathrm{pr}, i}\bigr)\, \mathrm{d}\mathbf x_0 \\
    &= \sum_{i = 1}^N p_i\, \mathcal{N}\bigl(\mathbf x_t; \sqrt{\bar\alpha(t)}\, \mathbf m_{\mathrm{pr}, i},\, \bar\alpha(t)\mathbf C_{\mathrm{pr}, i} + (1-\bar\alpha(t))\mathbf I_n\bigr).
\end{align}
By \Cref{lem:gmm-score},
\begin{equation}
    \nabla_{\mathbf x_t} \log p_t(\mathbf x_t)
    = - \sum_{i = 1}^N \tilde p_i(\mathbf x_t, t)\, \bigl(\bar\alpha(t) \mathbf C_{\mathrm{pr}, i} + (1-\bar\alpha(t))\mathbf I_n\bigr)^{-1}\bigl(\mathbf x_t - \sqrt{\bar\alpha(t)} \mathbf m_{\mathrm{pr}, i}\bigr),
\end{equation}
\begin{equation}
    \tilde p_i(\mathbf x_t, t) = \frac{p_i\, \mathcal{N}\bigl(\mathbf x_t; \sqrt{\bar\alpha(t)} \mathbf m_{\mathrm{pr}, i}, \bar\alpha(t) \mathbf C_{\mathrm{pr}, i} + (1-\bar\alpha(t))\mathbf I_n\bigr)}{\sum_{j = 1}^N p_j\, \mathcal{N}\bigl(\mathbf x_t; \sqrt{\bar\alpha(t)} \mathbf m_{\mathrm{pr}, j}, \bar\alpha(t) \mathbf C_{\mathrm{pr}, j} + (1-\bar\alpha(t))\mathbf I_n\bigr)}.
\end{equation}
By Bayes's rule, the denoiser is a Gaussian mixture with the same weights as the marginal,
\begin{equation}
    p_{0 \mid t}(\mathbf x_0 \mid \mathbf x_t) = \sum_{i = 1}^N \tilde p_i(\mathbf x_t, t)\, \mathcal{N}\bigl(\mathbf x_0; \mathbf m_{0 \mid t, i}(\mathbf x_t), \mathbf C_{0 \mid t, i}\bigr),
\end{equation}
\begin{align}
    \mathbf C_{0 \mid t, i} &= \Bigl(\mathbf C_{\mathrm{pr}, i}^{-1} + \tfrac{\bar\alpha(t)}{1-\bar\alpha(t)}\mathbf I_n\Bigr)^{-1}, \\
    \mathbf m_{0 \mid t, i}(\mathbf x_t) &= \mathbf C_{0 \mid t, i} \Bigl(\mathbf C_{\mathrm{pr}, i}^{-1}\mathbf m_{\mathrm{pr}, i} + \tfrac{\sqrt{\bar\alpha(t)}}{1 - \bar\alpha(t)}\mathbf x_t\Bigr).
\end{align}
By the laws of total expectation and total covariance over the component label,
\begin{align}
    \mathbf m_{0 \mid t}(\mathbf x_t) &= \sum_{i = 1}^N \tilde p_i(\mathbf x_t, t)\, \mathbf m_{0 \mid t, i}(\mathbf x_t), \\
    \mathbf C_{0 \mid t}(\mathbf x_t) &= \sum_{i = 1}^N \tilde p_i(\mathbf x_t, t)\, \mathbf C_{0 \mid t, i} + \frac{1}{2}\sum_{i, j = 1}^N \tilde p_i\, \tilde p_j\, \bigl(\mathbf m_{0 \mid t, i} - \mathbf m_{0 \mid t, j}\bigr)\bigl(\mathbf m_{0 \mid t, i} - \mathbf m_{0 \mid t, j}\bigr)^{\!\top}.
\end{align}
As in the discrete case, the second-moment decomposition makes the structure of $\mathbf C_{0 \mid t}$ transparent without reference to the aggregate mean: the first sum is the within-component covariance, the second sum is the pairwise outer-product spread over component means.

\paragraph{Conditional objects.}
The likelihood reduces to a mixture of integrals,
\begin{equation}
    p_{\mathbf y \mid t}(\mathbf y \mid \mathbf x_t) = \sum_{i = 1}^N \tilde p_i(\mathbf x_t, t) \int_{\mathcal X} \mathcal{N}\bigl(\mathbf y; \mathcal{A}(\mathbf x_0), \bm\Sigma_{\mathbf y}\bigr)\, \mathcal{N}\bigl(\mathbf x_0; \mathbf m_{0 \mid t, i}, \mathbf C_{0 \mid t, i}\bigr) \mathrm{d}\mathbf x_0,
\end{equation}
each of which is intractable for general $\mathcal{A}$. For affine $\mathcal{A}(\mathbf x_0) = \mathbf A \mathbf x_0 + \mathbf b$, each component integrates via \Cref{lem:gaussian-gaussian-marginal},
\begin{equation}
    p_{\mathbf y \mid t}(\mathbf y \mid \mathbf x_t) = \sum_{i = 1}^N \tilde p_i(\mathbf x_t, t)\, \mathcal{N}\bigl(\mathbf y; \mathbf A\mathbf m_{0 \mid t, i} + \mathbf b,\, \bm\Sigma_{\mathbf y} + \mathbf A\mathbf C_{0 \mid t, i}\mathbf A^\top\bigr).
\end{equation}

For the posterior we again work with the joint $(\mathbf x_t, \mathbf y)$. By \Cref{lem:gaussian-gmm-marginal} applied componentwise,
\begin{equation}
    p(\mathbf x_t, \mathbf y) = \sum_{i = 1}^N p_i\, \mathcal{N}\!\left(\begin{pmatrix}\mathbf x_t\\\mathbf y\end{pmatrix}; \begin{pmatrix}\sqrt{\bar\alpha(t)}\,\mathbf m_{\mathrm{pr}, i} \\ \mathbf A\mathbf m_{\mathrm{pr}, i} + \mathbf b\end{pmatrix},\, \bm\Lambda_i\right),
\end{equation}
\begin{equation}
    \bm\Lambda_i = \begin{pmatrix}\bar\alpha(t)\mathbf C_{\mathrm{pr}, i} + (1-\bar\alpha(t))\mathbf I_n & \sqrt{\bar\alpha(t)}\,\mathbf C_{\mathrm{pr}, i}\mathbf A^\top \\ \sqrt{\bar\alpha(t)}\,\mathbf A\mathbf C_{\mathrm{pr}, i} & \mathbf A\mathbf C_{\mathrm{pr}, i}\mathbf A^\top + \bm\Sigma_{\mathbf y}\end{pmatrix}.
\end{equation}
Conditioning each component on $\mathbf y$ and reweighting by its $\mathbf y$-marginal likelihood (\Cref{lem:gaussian-gmm-conjugacy}) yields a Gaussian mixture posterior
\begin{equation}
    p_{t \mid \mathbf y}(\mathbf x_t \mid \mathbf y) = \sum_{i = 1}^N \pi_i(\mathbf y)\, \mathcal{N}\bigl(\mathbf x_t; \mathbf m_{\mathrm{post}, i}(\mathbf y), \mathbf C_{\mathrm{post}, i}\bigr),
\end{equation}
\begin{align}
    \pi_i(\mathbf y) &= \frac{p_i\, \mathcal{N}\bigl(\mathbf y; \mathbf A\mathbf m_{\mathrm{pr}, i} + \mathbf b,\, \mathbf A\mathbf C_{\mathrm{pr}, i}\mathbf A^\top + \bm\Sigma_{\mathbf y}\bigr)}{\sum_{j = 1}^N p_j\, \mathcal{N}\bigl(\mathbf y; \mathbf A\mathbf m_{\mathrm{pr}, j} + \mathbf b,\, \mathbf A\mathbf C_{\mathrm{pr}, j}\mathbf A^\top + \bm\Sigma_{\mathbf y}\bigr)}, \\
    \mathbf m_{\mathrm{post}, i}(\mathbf y) &= \sqrt{\bar\alpha(t)}\,\mathbf m_{\mathrm{pr}, i} + \sqrt{\bar\alpha(t)}\,\mathbf C_{\mathrm{pr}, i}\mathbf A^\top\bigl(\mathbf A\mathbf C_{\mathrm{pr}, i}\mathbf A^\top + \bm\Sigma_{\mathbf y}\bigr)^{-1}(\mathbf y - \mathbf A\mathbf m_{\mathrm{pr}, i} - \mathbf b), \\
    \mathbf C_{\mathrm{post}, i} &= \bar\alpha(t)\mathbf C_{\mathrm{pr}, i} + (1-\bar\alpha(t))\mathbf I_n - \bar\alpha(t)\mathbf C_{\mathrm{pr}, i}\mathbf A^\top\bigl(\mathbf A\mathbf C_{\mathrm{pr}, i}\mathbf A^\top + \bm\Sigma_{\mathbf y}\bigr)^{-1}\mathbf A\mathbf C_{\mathrm{pr}, i}.
\end{align}
By \Cref{lem:gmm-score} applied to this posterior mixture,
\begin{equation}
    \nabla_{\mathbf x_t}\log p_{t \mid \mathbf y}(\mathbf x_t \mid \mathbf y) = -\sum_{i = 1}^N \pi'_i(\mathbf x_t, \mathbf y)\, \mathbf C_{\mathrm{post}, i}^{-1}\bigl(\mathbf x_t - \mathbf m_{\mathrm{post}, i}(\mathbf y)\bigr),
\end{equation}
\begin{equation}
    \pi'_i(\mathbf x_t, \mathbf y) = \frac{\pi_i(\mathbf y)\, \mathcal{N}\bigl(\mathbf x_t; \mathbf m_{\mathrm{post}, i}(\mathbf y), \mathbf C_{\mathrm{post}, i}\bigr)}{\sum_{j = 1}^N \pi_j(\mathbf y)\, \mathcal{N}\bigl(\mathbf x_t; \mathbf m_{\mathrm{post}, j}(\mathbf y), \mathbf C_{\mathrm{post}, j}\bigr)}.
\end{equation}
The likelihood score follows by subtracting the prior score from the posterior score.

\subsection{Experimental framework}
\label{app:setup-common}

All experiments use the variance-preserving SDE (VP-SDE) of
\Cref{eq:vpsde} with the default linear $\beta$-schedule.
Forward marginals, scores, and posteriors are evaluated on a uniform time grid of $N_t = 400$ points spanning $t \in [10^{-3},\, 1]$, and on a uniform spatial grid of $N_x = 600$ points whose extent is set per-prior to bracket its support. All computations are carried out in double precision (\texttt{torch.float64}) on CPU.

For all problems, the observation model is one-dimensional with additive Gaussian noise:
\begin{equation}
  y = \mathcal{A}(x_0) + \eta, \quad \eta \sim \mathcal{N}(0,\, \sigma_{\mathbf y}^2).
  \label{eq:exp-obs}
\end{equation}

\paragraph{Methods.}
For every $(t, x)$ grid we report up to six posterior distributions $p(x_t \mid y)$:
\begin{enumerate}
    \item the true posterior (where available), derived in \Cref{sec:target-objects},
    \item FSR: the finite-sample regime introduced in \Cref{sec:finite-sample-regime},
    \item $\sigma$-DPS \citep{chung2023diffusion}, defined in \Cref{eq:chung-score} with unmodified prefactor $\sigma_{\mathbf y}^{-2}$,
    \item $\zeta$-DPS \citep{chung2023diffusion}, defined in \Cref{eq:modified-chung-score} with modified prefactor $\zeta/\|\mathbf y - \mathcal{A}(\mathbf m_{0 \mid t}(\mathbf x_t))\|_2$,
    \item $\Pi$GDM \citep{song2023pseudoinverseguided}, defined in \Cref{eq:song-approximation}, and 
    \item TMPD \citep{boys2024tweedie}, defined in \Cref{eq:boys-approximation}. 
\end{enumerate}
The FSR and DPS methods are present for all problems, while $\Pi$GDM and TMPD are only applicable when the measurement model is affine.

Heatmaps use power-law normalization with exponent $\gamma = 0.55$. The color ceiling is set per row to the $99$\textsuperscript{th} percentile of that row's analytic field with flipping at the $99.9$\textsuperscript{th} percentile to prevent Dirac distributions from dominating the colormap.

\subsection{Per-prior posterior reconstructions}
\label{sec:per-prior-recons}

This subsection presents posterior-score reconstructions for every $(\text{prior}, A, \mathbf y)$ target in the testbed. Each prior is introduced with its mathematical definition and then exercised across its three forward operators and three measurements. Per target we report all six methods of \Cref{app:setup-common}; on non-affine $A$ the Song and Boys cells are marked \emph{intractable} per the linear-only restriction. The $\zeta$-DPS column displays the per-target $\zeta^\star$ chosen by visual inspection from the $\zeta \in \{0.01, 0.03, \dots, 0.49\}$ tuning sweep.

\clearpage
\subsubsection{Discrete prior \texttt{tri\_equal}}
\label{sec:setup-discrete-tri-equal}

Symmetric three-atom prior, equal weights.
\begin{equation*}
  p_{\mathrm{pr}}(\mathbf x_0) = \tfrac{1}{3}\,\delta(\mathbf x_0 + 1.8) + \tfrac{1}{3}\,\delta(\mathbf x_0 - 0.2) + \tfrac{1}{3}\,\delta(\mathbf x_0 - 2.2)
\end{equation*}

\begin{equation*}
    \mathcal{A}(\mathbf x) = \mathbf x, \quad \sigma = 0.3, \quad \mathbf y = -1.8000
\end{equation*}
\begin{figure}[H]
  \centering

\end{figure}

\clearpage
\subsubsection{Discrete prior \texttt{pent\_asym}}
\label{sec:setup-discrete-pent-asym}

Five-atom prior with mildly asymmetric tails.
\begin{equation*}
  p_{\mathrm{pr}}(\mathbf x_0) = 0.10\,\delta(\mathbf x_0 + 2.7) + 0.25\,\delta(\mathbf x_0 + 0.7) + 0.30\,\delta(\mathbf x_0 - 0.3) + 0.25\,\delta(\mathbf x_0 - 1.3) + 0.10\,\delta(\mathbf x_0 - 3.3)
\end{equation*}

\begin{equation*}
    \mathcal{A}(\mathbf x) = \mathbf x, \quad \sigma = 0.3, \quad \mathbf y = -2.7000
\end{equation*}
\begin{figure}[H]
  \centering
  %
\end{figure}

\clearpage
\subsubsection{Discrete prior \texttt{wild}}
\label{sec:setup-discrete-wild}

Five-atom prior concentrating $80\%$ of mass on a tight negative cluster.
\begin{equation*}
  p_{\mathrm{pr}}(\mathbf x_0) = 0.05\,\delta(\mathbf x_0 + 4.0) + 0.50\,\delta(\mathbf x_0 + 1.2) + 0.30\,\delta(\mathbf x_0 + 0.8) + 0.10\,\delta(\mathbf x_0 - 2.5) + 0.05\,\delta(\mathbf x_0 - 5.5)
\end{equation*}

\begin{equation*}
    \mathcal{A}(\mathbf x) = \mathbf x, \quad \sigma = 0.3, \quad \mathbf y = -4.0000
\end{equation*}
\begin{figure}[H]
  \centering
  %
\end{figure}

\clearpage
\subsubsection{Gaussian prior \texttt{narrow}}
\label{sec:setup-gaussian-narrow}

Unimodal Gaussian with small variance, displaced from zero so the forward marginal drifts visibly toward the unit Gaussian as $t \to 1$.
\begin{equation*}
  p_{\mathrm{pr}}(\mathbf x_0) = \mathcal{N}(0.6,\ 0.5)
\end{equation*}

\begin{equation*}
    \mathcal{A}(\mathbf x) = \mathbf x, \quad \sigma = 0.3, \quad \mathbf y = -2.0000
\end{equation*}
\begin{figure}[H]
  \centering
  %
\end{figure}

\clearpage
\subsubsection{Gaussian prior \texttt{wide}}
\label{sec:setup-gaussian-wide}

Unimodal Gaussian with large variance, displaced from zero.
\begin{equation*}
  p_{\mathrm{pr}}(\mathbf x_0) = \mathcal{N}(1.5,\ 2.0)
\end{equation*}

\begin{equation*}
    \mathcal{A}(\mathbf x) = \mathbf x, \quad \sigma = 0.3, \quad \mathbf y = -2.0000
\end{equation*}
\begin{figure}[H]
  \centering
  %
\end{figure}

\clearpage
\subsubsection{Gaussian-mixture prior \texttt{tri\_equal}}
\label{sec:setup-gmm-tri-equal}

Symmetric three-component Gaussian mixture, equal weights and equal isotropic variances.
\begin{equation*}
  p_{\mathrm{pr}}(\mathbf x_0) = \tfrac{1}{3}\,\mathcal{N}(-2.6,\ 0.25) + \tfrac{1}{3}\,\mathcal{N}(0.4,\ 0.25) + \tfrac{1}{3}\,\mathcal{N}(3.4,\ 0.25)
\end{equation*}

\begin{equation*}
    \mathcal{A}(\mathbf x) = \mathbf x, \quad \sigma = 0.2, \quad \mathbf y = -2.6000
\end{equation*}
\begin{figure}[H]
  \centering
  %
\end{figure}

\clearpage
\subsubsection{Gaussian-mixture prior \texttt{bi\_asym}}
\label{sec:setup-gmm-bi-asym}

Bimodal Gaussian mixture with asymmetric weights and component variances.
\begin{equation*}
  p_{\mathrm{pr}}(\mathbf x_0) = 0.30\,\mathcal{N}(-1.7,\ 0.16) + 0.70\,\mathcal{N}(2.3,\ 0.36)
\end{equation*}

\begin{equation*}
    \mathcal{A}(\mathbf x) = \mathbf x, \quad \sigma = 0.2, \quad \mathbf y = -1.7000
\end{equation*}
\begin{figure}[H]
  \centering
  %
\end{figure}

\section{$\zeta$-DPS tuning}
\label[appendix]{sec:zeta-tuning}
\input{E-zeta-DPS-tuning}



\end{document}